\documentclass[journal]{IEEEtran}
\usepackage{cite}
\usepackage{amsthm}
\usepackage{times}
\usepackage{epsfig}
\usepackage{graphicx}
\usepackage{amsmath}
\usepackage{amssymb}
\usepackage{algorithm}
\usepackage{algpseudocode}
\usepackage{amsmath}
\usepackage{amssymb}
\usepackage{bm}
\usepackage{booktabs}
\usepackage{color}
\usepackage{epsfig}
\usepackage{epstopdf}
\usepackage{subfigure}
\usepackage{colortbl}
\usepackage{arydshln}
\usepackage{float}
\usepackage{graphicx}
\DeclareGraphicsExtensions{.pdf,.jpeg,.png}
\usepackage{ifthen}
\usepackage{multirow}
\usepackage[square, comma, sort&compress, numbers]{natbib}
\usepackage{makecell} 
\usepackage{url}  

\usepackage{pifont}
\usepackage{flafter}
\usepackage{times}
\usepackage{threeparttable}
\usepackage{xcolor}
\usepackage{arydshln}
\definecolor{darkgreen}{rgb}{0.0, 0.2, 0.13}
\definecolor{greed}{rgb}{0.0, 0.5, 0.0}

\renewcommand{\algorithmicrequire}{\textbf{Input:}} 

\newtheorem{lemma}{\textbf{Lemma}}
\newcommand{\argmin}{\mathop{\mathrm{arg\,min}}}

\newtheorem{thm}{\textbf{Theorem}}
\newtheorem{mydef}{\textbf{Definition}}
\newcommand{\red}{\color{black}}

\begin{document}
%
\title{Dictionary Learning with Low-rank Coding Coefficients for Tensor
Completion}
%
%
%

\author{Tai-Xiang~Jiang,
        Xi-Le~Zhao,
        Hao~Zhang,
        \and Michael~K.~Ng\IEEEauthorrefmark{1} 
        \thanks{\IEEEauthorrefmark{1} Corresponding author.}
        \thanks{This work is supported in part by the Fundamental Research Funds for the Central Universities (JBK2102001),
        in part by the National Natural Science Foundation of China (12001446, 61772003, and 61876203), in part by the HKRGC  (GRF
        12200317, 12300218, 12300519, and 17201020).}
         \thanks{T.-X. Jiang is with FinTech Innovation Center, School of Economic Information Engineering, Southwestern University of Finance and Economics, Chengdu, Sichuan, P.R.China (e-mail: taixiangjiang@gmail.com, jiangtx@swufe.edu.cn).}
        \thanks{X.-L. Zhao and H. Zhang are with Research Center for Image and Vision Computing, School of Mathematical Sciences, University of Electronic Science and Technology of China, Chengdu 611731, P.R.China (e-mail: xlzhao122003@163.com; 201821110218@std.uestc.edu.cn).}
        \thanks{M. K. Ng is with Department of Mathematics, The University of Hong Kong, Pokfulam, Hong Kong (e-mail: mng@maths.hku.hk).
}}


\maketitle

\begin{abstract}
In this paper, we propose a novel tensor learning and coding model for third-order data completion.
Our model is to learn a data-adaptive dictionary from the given observations, and determine the coding coefficients of third-order tensor tubes.
In the completion process, we minimize the low-rankness of each tensor slice containing the coding coefficients.
By comparison with the traditional pre-defined transform basis, the advantages of the proposed model are that (i) the dictionary can be learned based on the given data observations so that the basis can be more
adaptively and accurately constructed, and (ii) the low-rankness of the coding coefficients can allow
the linear combination of dictionary features more effectively.
Also we develop a multi-block proximal alternating minimization algorithm for solving
such tensor learning and coding model, and show
that the sequence generated by the algorithm can globally converge to a critical point.
Extensive experimental results for real data sets such as videos, hyperspectral images, and traffic data are reported to demonstrate these advantages and show the performance of the proposed tensor learning and coding method is significantly better than the other tensor completion methods in terms of several evaluation metrics.
\end{abstract}


\IEEEpeerreviewmaketitle

\section{Introduction}\label{Sec-Intro}
\IEEEPARstart{T}{ensor} completion is a problem of filling the missing or unobserved entries of the incomplete observed data, playing an important role in a wide range of real-world applications, such as color image inpainting \cite{Bertalmio2000imInpait,Komodakis2006Global,Korah2007TIP,Liu2013PAMItensor,zhao2020deep}, high-speed compressive video \cite{koller2015high}, magnetic resonance imaging (MRI) data recovery \cite{MRITV}, and hyperspectral data inpainting \cite{zhuang2018fast}. Generally, many real-world tensors are inner correlated, the spectral redundancy \cite{bioucas2013hyperspectral} of the hyperspectral images (HSIs) for example. Therefore, it is effective to utilize the global low-dimensional structure to characterize the relationship between the missing entries and observed ones.

\begin{figure}[!t]
\setlength{\tabcolsep}{3pt}
\renewcommand\arraystretch{1}
\centering
\includegraphics[width=0.999\linewidth]{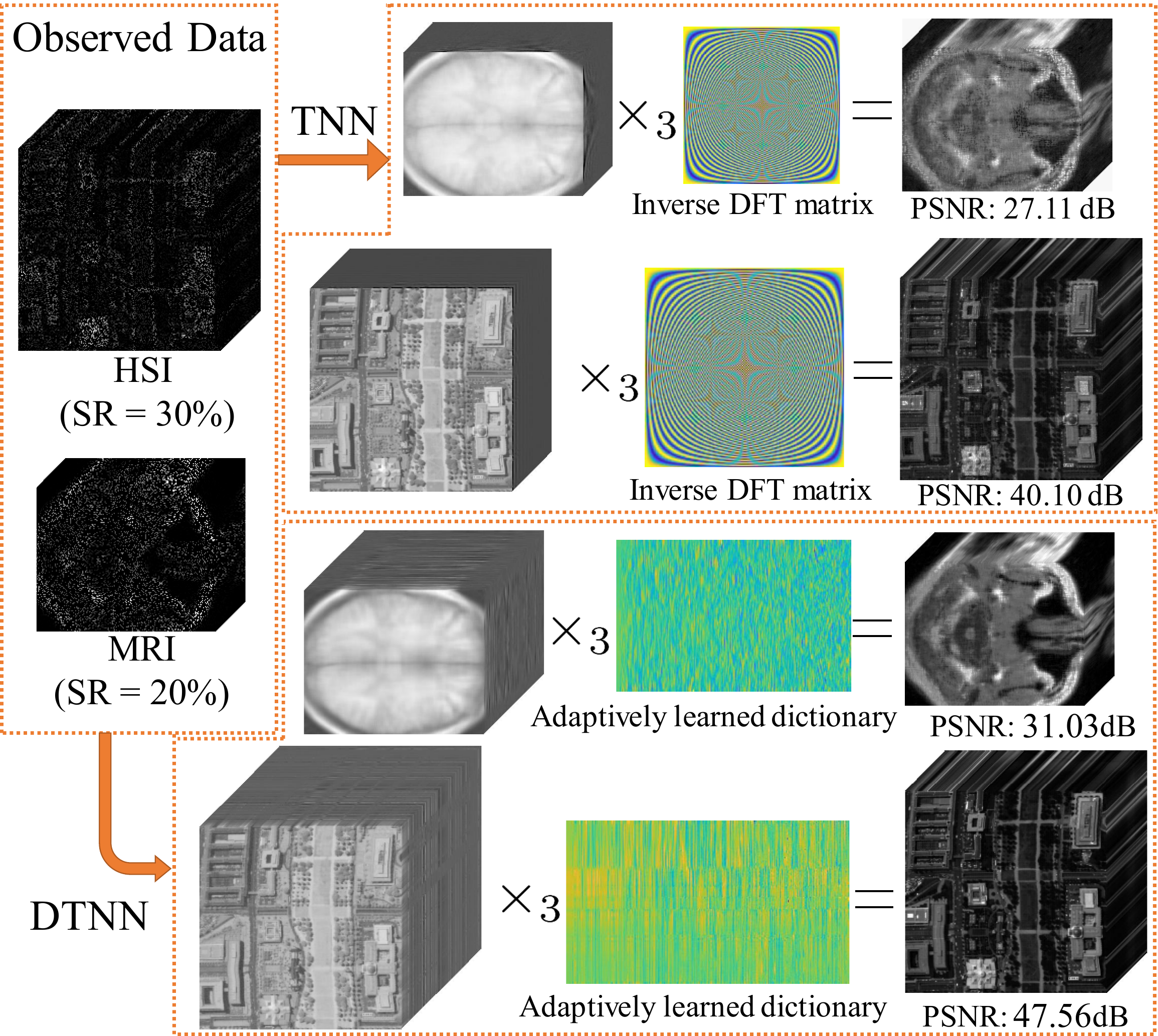}
\caption{An illustration of the TNN based LRTC and the DTNN based LRTC. SR denotes the sampling rate.}
\label{fig-flowchart}
\end{figure}

Generally, like the matrix case, the low-rank tensor completion (LRTC) can be formulated as
\begin{equation}
\min \ \text{rank}(\mathcal{X})\quad\text{s.t.} \ \mathcal{X}_\Omega=\mathcal{O}_\Omega,
\end{equation}
where $\mathcal{X}$ is the underlying tensor, $\mathcal{O}$ is the observed incomplete tensor as shown in the top-left of Fig. \ref{fig-flowchart}, $\Omega$ is the index set corresponding to the observed entries, and $\mathcal{X}_\Omega=\mathcal{O}_\Omega$ enforces the entries of $\mathcal{X}$ in $\Omega$ equal to the observation $\mathcal{O}$.
However, unlike the matrix cases, the definition of the tensor rank is still not unique and has received considerable attentions in recent researches.
{\red Generally, different definitions of the tensor rank are respectively based on different tensor decomposition schemes.
For instance, the CANDECOMP/PARAFAC (CP)-rank, based on the CP decomposition, is defined as the minimal rank-one tensors to express the original data \cite{kiers2000towards}. Although determinating the CP-rank of a given tensor is NP-hard \cite{hillar2013most}, CP decomposition have been successfully applied for tensor recovery problem \cite{zhao2015bayesian,han2018generalized,zhao2016bayesian}.
The Tucker-rank, corresponding to the Tucker decomposition \cite{tucker1966some}, is defined as a vector constituted of the ranks of the unfolding matrices along all modes. Liu {\em et al.} \cite{Liu2013PAMItensor} propose a convex surrogate of the Tucker-rank and minimize it for the LRTC problem while Zhang {\em et al.} \cite{zhang2018nonconvex} resort to use a family of nonconvex functions onto the singular values.
Another newly emerged one is the tensor train (TT)-rank derived from the TT decomposition \cite{oseledets2011tensor}. In this framework, the tensor is decomposed in a chain manner with nodes being third-order tensors. Bengua {\em et al.} \cite{bengua2017efficient} minimize a nuclear norm based on the TT-rank  for the color image and video recovery. The TT-rank has also been applied for the HSI super-resolution \cite{dian2019learning} and the tensor-on-tensor regression \cite{liu2020low}. When factors cyclically connected, it becomes the tensor ring (TR) decomposition \cite{zhao2016tensor}.
Yuan {\em et al.} \cite{yuan2019tensor} exploit the low-rank structure of the TR latent space and regularize the latent TR factors with the nuclear norm. Yu {\em et al.} \cite{yu2020low} introduce the tensor circular unfolding for the TR decomposition and perform parallel low-rank matrix factorizations to all circularly unfolded matrices for tensor completion. Please refer to \cite{long2019low,song2019tensor} for a comprehensive overview of the LRTC problem.}

This work fixes attentions on novel notions of the tensor rank, i.e., the tensor tubal-rank and multi-rank, which are derived from the tensor singular value decomposition (t-SVD) framework \cite{braman2010third, kilmer2011factorization,hao2013facial}.
{\red
The t-SVD framework is constructed based on a fundamental tensor-tensor product (t-prod) operation (see Def. \ref{Def:1}), which is closed on the set of third-order tensors and allows tensor factorizations which are analogs of matrix factorizations such as SVD. Meanwhile, it further allows new extensions of familiar matrix analysis to the multilinear setting while avoiding the loss of information inherent in matricization or flattening of the third-order tensor \cite{kilmer2013third}.}
For a third-order tensor $\mathcal{X}\in\mathbb{R}^{n_1\times n_2\times n_3}$, its t-SVD is given as $\mathcal{X} = \mathcal{U}*\mathcal{S}*\mathcal{V}^\text{H}$, {\red where $\mathcal{U}\in\mathbb{R}^{n_1\times n_1\times n_3}$ and $\mathcal{V}\in\mathbb{R}^{n_2\times n_2\times n_3}$ are orthogonal tensors, $\mathcal{S}\in\mathbb{R}^{n_1\times n_2\times n_3}$ is an f-diagonal tensor (see Def. \ref{Def:3}),} and $*$ denotes the t-prod (see Def. \ref{Def:1}). The tensor tubal-rank of $\mathcal{X}$ is defined as the number of non-zero singular tubes of $\mathcal{S}$. Since that the LRTC problem associated with the tensor tubal-rank (or multi-rank) is NP-hard, Zhang {\em et al.} \cite{zhang2014novel} turn to minimize the tensor nuclear norm (TNN, see Def. \ref{Def-tnn}), {\red which is a convex envelope of the $\ell_1$ norm of the tensor multi-rank}, and they establish the theoretical guarantee in \cite{zhang2017exact}. Jiang {\em et al.} \cite{jiang2019robust} and Wang {\em et al.} \cite{wang2019robust} tackle the robust tensor completion task, in which the incomplete observations are corrupted by sparse outliers, via minimizing TNN.
The TNN based LRTC model is given as
\begin{equation}\label{TNN-LRTC}
\min \ \|\mathcal{X}\|_\text{TNN}\quad \text{s.t.} \ \mathcal{X}_\Omega=\mathcal{O}_\Omega.
\end{equation}
{\red As the t-prod is based on a convolution-like operation,
the computation of t-prod and TNN could be implemented with the discrete Fourier transform (DFT) or the fast Fourier transform (FFT).}

In \cite{kernfeld2015tensor}, Kernfeld {\em et al.} further note that a more general tensor-tensor product could be defined with any invertible linear transforms.
Also, the TNN in \eqref{TNN-LRTC} can be alternatively constructed using other transform, e.g., the discrete cosine transform (DCT) adopted by Lu  {\em et al.}  \cite{lu2019low} and Xu {\em et al.} \cite{xu2019fast}, and the Haar wavelet transform exploited in \cite{song2019robust}. Furthermore, Jiang {\em et al.} \cite{jiang2019framelet} introduce the framelet transform, which is semi-invertible, and break through the restriction of invertibility.
%
Within these transform based TNN methods, the typical pipeline is applying one selected transform along the third dimension, and minimizing the low-rankness of slices of the transformed data for the completion.
Once the tubes of the original
tensor are highly correlated, the frontal slices of the transformed data would be low-rank \cite{jiang2019framelet,song2019robust}.

An unavoidable issue is that the correlations along the third mode are different for various types of data. For example, the redundancy of HSIs along the third mode are much higher than videos with changing scenes. Thus, predefined transforms usually lack flexibility and could not be suitable for all kinds of data. Therefore, to address this issue, we construct a dictionary, which can be adaptively inferred form the data, instead of inverse transforms mentioned above.
%
As mentioned by the authors of \cite{lu2019low,song2019robust}, it is interesting to learn the transform for implementing the t-SVD from the data in different tasks. Our approach can be viewed as learning the inverse transform from this perspective and indeed enriches the research on this topic. The methods, which utilize DFT, DCT, and Framelet, can be viewed as specific instances of our method with fixed dictionaries, i.e., the inverse discrete transformation matrices.
{From the view of dictionary learning, our method can also be interpreted as learning a dictionary with low-rank coding.
We enforce the low-rankness of the coding coefficients in a tensor manner, and this allows the linear combination of features, namely, the atoms of the dictionary.}

The main contributions of this paper mainly consist of three aspects:
\begin{itemize}
  \item We propose novel tensor learning and coding model, which is to adaptively learn a dictionary from the observations
and determine the low rank coding coefficients, for the third-order tensor completion.
  \item A multi-block proximal alternating minimization algorithm is designed to solve the proposed non-convex model. We theoretically prove its global convergence to a critical point.
  \item Extensive experiments are conducted on various types of real-world third-order tensor data. The results illustrate that our method outperforms compared LRTC methods.
\end{itemize}

This paper is organized as follows. Sec. \ref{Sec-Pre} introduces related works and the basic preliminaries. Our method is given in Sec. \ref{Sec:Model}. We report the experimental results in Sec. \ref{Sec-Exp}. Finally, Sec. \ref{Sec-Con} draws some conclusions.
\section{Preliminaries}\label{Sec-Pre}
%

Throughout this paper, lowercase letters, e.g., $x$, boldface lowercase letters, e.g., $\mathbf x$, boldface upper-case letters, e.g., $\mathbf X$, and boldface calligraphic letters,  e.g., $\mathbf{\mathcal{X}}$, are used to denote scalars, vectors, matrices, and tensors, respectively.
Given a third-order tensor $\mathbf{\mathcal{X}}\in \mathbb{R}^{n_{1}\times n_2\times n_{3}}$, we use $\mathcal X_{ijk}$ to denote its $(i,j,k)$-th element.
The $k$-th frontal slice of $\mathbf{\mathcal{X}}$ is denoted as $\mathcal{X}^{(k)}$ (or $\mathcal{X}(:,:,k)$, $\mathbf X^k$), and the mode-3 unfolding matrix of $\mathbf{\mathcal{X}}$ is denoted as $\mathbf X_{(3)}\in \mathbb{R}^{n_3\times n_1n_2}$. We use ${\tt fold}_3$ and ${\tt unfold}_3$ to denote the folding and unfolding operations along the third dimension, respectively, and we have $\mathcal{X}={\tt fold}_{3}({\tt unfold}_{3}(\mathcal{X})) = {\tt fold}_{3}(\mathbf X_{(3)})$.
The mode-3 tensor-matrix product is denoted as $\times_3$, and we have $\mathcal X\times_3\mathbf A\Leftrightarrow \mathbf A {\tt unfold}_{3}(\mathcal{X})$.
{\red The DFT matrix and inverse DFT matrix for a vector of the length $n$ are respectively denoted as $\mathbf{F}_{n}$ and $\mathbf{F}_{n}^{-1}$. For the tensor $\mathbf{\mathcal{X}}\in \mathbb{R}^{n_{1}\times n_2\times n_{3}}$,
its Fourier transformed (along the third mode) tensor $\mathcal{Z}\in\mathbb{C}^{n_1\times n_2\times n_3}$ can be obtained by $\mathcal{Z} = \mathcal{X}\times_3\mathbf{F}_{n_3}$, and we have $\mathcal{X} = \mathcal{Z}\times_3\mathbf{F}_{n_3}^{-1}$.}
The tensor Frobenius norm of a third-order tensor $\mathcal{X}$ is defined as $\left\|\mathcal{X}\right\|_{F}:=\sqrt{\langle\mathcal{X},\mathcal{X}\rangle} = \sqrt{\sum_{ijk}\mathcal X_{ijk}^2}$. {\red For a matrix $\mathbf X\in\mathbb{C}^{n_1\times n_2}$, its matrix nuclear norm is denoted as $\|\mathbf X\|_*=\sum_{i=1}^{\min\{n_1,n_2\}} \sigma_i(\mathbf{X})$, where $\sigma_i(\mathbf{X})$ is the $i$-th largest singular value of $\mathbf X$.}

\begin{mydef}[tensor conjugate transpose \cite{kilmer2013third}]
The conjugate transpose of a tensor $\mathbf{\mathcal{A}}\in \mathbb{C}^{n_{1}\times n_2\times n_{3}}$ is tensor $\mathbf{\mathcal{A}}^\text{\rm H}\in \mathbb{C}^{n_{2}\times n_1\times n_{3}}$ obtained by conjugate transposing each of the frontal slice and then reversing the order of transposed frontal slices 2 through $n_3$, i.e., $
\left(\mathbf{\mathcal{A}}^\text{\rm H}\right)^{(1)}=\left(\mathbf{\mathcal{A}}^{(1)}\right)^\text{\rm H}$ and $\left(\mathbf{\mathcal{A}}^\text{\rm H}\right)^{(i)}=\left(\mathbf{\mathcal{A}}^{(n_3+2-i)}\right)^\text{\rm H}$  for $i=2,\cdots,n_3$.
\end{mydef}


\begin{mydef}[t-prod \cite{kilmer2013third}]\label{Def:1}
The tensor-tensor-product (t-prod) $\mathbf{\mathcal{C}}=\mathbf{\mathcal{A}}*\mathbf{\mathcal{B}}$
of $\mathbf{\mathcal{A}}\in \mathbb{R}^{n_{1}\times n_2\times n_{3}}$ and $\mathbf{\mathcal{B}}\in \mathbb{R}^{n_{2}\times n_4\times n_{3}}$ is a tensor of size
$n_1\times n_4 \times n_3$, where the $(i,j)$-th tube $\mathbf{c}_{ij:}$ is given by
\begin{equation}
\mathbf{c}_{ij:} = \mathbf{\mathcal{C}}(i,j,:) = \sum_{k=1}^{n_2}\mathbf{\mathcal{A}}(i,k,:)\circledast\mathbf{\mathcal{B}}(k,j,:)
\end{equation}
where $\circledast$ denotes the circular convolution between two tubes of same size.
\label{def:tprod}
\end{mydef}{\red
Equivalently, for $\mathbf{\mathcal{C}}=\mathbf{\mathcal{A}}*\mathbf{\mathcal{B}}$, we have
\begin{equation*}
\begin{aligned}\left(
\begin{matrix}
  \mathcal{C}^{(1)}  \\
  \mathcal{C}^{(2)} \\
  \vdots    \\
  \mathcal{C}^{(n_3)}\\
\end{matrix}\right)
=
\left(
\begin{matrix}
  \mathcal{A}^{(1)} & \mathcal{A}^{(n_3)} & \cdots & \mathcal{A}^{(2)} \\
  \mathcal{A}^{(2)} & \mathcal{A}^{(1)} & \cdots & \mathcal{A}^{(3)} \\
  \vdots &  \vdots & \ddots &  \vdots \\
  \mathcal{A}^{(n_3)} & \mathcal{A}^{(n_3-1)} & \cdots & \mathcal{A}^{(1)}
\end{matrix}
\right)
\left(
\begin{matrix}
  \mathcal{B}^{(1)}  \\
  \mathcal{B}^{(2)} \\
  \vdots   \\
  \mathcal{B}^{(n_3)} \\
\end{matrix}
\right),
\end{aligned}
\end{equation*}
where the first item in the right part of the equation is also called the block circulant
unfolding of $\mathcal{A}$.}

{\red
\begin{mydef}[Face-wise product \cite{kernfeld2015tensor}]
For two third-order tensors $\mathcal{A}\in \mathbb{R}^{n_{1}\times n_2\times n_{3}}$ and $\mathcal{B}\in \mathbb{R}^{n_{2}\times n_4\times n_{3}}$, their face-wise product $\mathcal{A}\triangle\mathcal{B}\in\mathbb{R}^{n_1\times n_4\times n_3}$ is defined according to
$$
\left(\mathcal{A}\triangle\mathcal{B}\right)^{(i)} = \mathcal{A}^{(i)}\mathcal{B}^{(i)}, \ \text{for} \ i = 1,2,\cdots, n_3.
$$
\label{def:faceprod}
\end{mydef}

As the convolution operation could be converted to element-wise product via the Fourier transform, we have
\begin{equation}\label{def-t-prod-f}
\mathcal{C} =\mathcal{A}\ast\mathcal{B} =\left(\left(\mathcal{A}\times_3\mathbf F_{n_3}\right) \triangle\left(\mathcal{B} \times_3\mathbf F_{n_3}\right)\right)\times_3\mathbf{F}_{n_3}^{-1}.
\end{equation}
Eq. \eqref{def-t-prod-f} indicates that we can compute the t-prod between two tensors using the DFT matrix or the fast Fourier transform (FFT) for acceleration.
}


\begin{mydef}[special tensors \cite{kilmer2013third}]\label{Def:3}
The {\bf identity} tensor $\mathbf{\mathcal{I}}\in \mathbb{R}^{n_{1}\times n_1\times n_{3}}$ is the tensor whose first frontal slice is the $n_1\times n_1$ identity matrix, and whose other frontal slices are all
zeros.
A tensor $\mathbf{\mathcal{Q}} \in \mathbb{C}^ {n_{1} \times n_1\times n_{3}}$  is {\bf orthogonal} if it satisfies
\begin{equation}
\mathbf{\mathcal{Q}}^\text{\rm H}*\mathbf{\mathcal{Q}}=\mathbf{\mathcal{Q}}*\mathbf{\mathcal{Q}}^\text{\rm H}=\mathbf{\mathcal{I}}.
\end{equation}
A tensor $\mathbf{\mathcal{A}}$ is called {\bf f-diagonal} if each frontal slice $\mathbf{\mathcal{A}}^{(i)}$ is a diagonal matrix.
\end{mydef}

\begin{thm}[t-SVD \cite{kilmer2013third,kilmer2011factorization}]
For $\mathbf{\mathcal{A}}\in \mathbb{R}^{n_{1}\times n_2\times n_{3}}$, the t-SVD of $\mathbf{\mathcal{A}}$ is given by
\begin{equation}
\mathbf{\mathcal{A}}=\mathbf{\mathcal{U}}*\mathbf{\mathcal{S}}*\mathbf{\mathcal{V}}^\text{\rm H}
\end{equation}
where $\mathbf{\mathcal{U}}\in \mathbb{R}^{n_{1}\times n_1\times n_{3}}$ and $\mathbf{\mathcal{V}}\in \mathbb{R}^{n_{2}\times n_2\times n_{3}}$ are orthogonal tensors, and $\mathbf{\mathcal{S}}\in \mathbb{R}^{n_{1}\times n_2\times n_{3}}$ is an f-diagonal tensor.
\end{thm}

{\red
\begin{mydef}[tensor tubal-rank \cite{zhang2014novel}]\label{Def:tubal}
The tensor tubal-rank of a tensor $\mathbf{\mathcal{A}}\in\mathbb{R}^{n_1\times n_2\times n_3}$, denoted as $\text{rank}_t(\mathbf{\mathcal{A}})$, is defined as the number of non-zero singular tubes in $\mathbf{\mathcal{S}}$, where $\mathbf{\mathcal{S}}$ is from the t-SVD of $\mathbf{\mathcal{A}}$: $\mathbf{\mathcal{A}}=\mathbf{\mathcal{U}}*\mathbf{\mathcal{S}}*\mathbf{\mathcal{V}}^\text{H}$. Formally, we can write
$$\text{rank}_t(\mathbf{\mathcal{A}})=\#\{i,\mathbf{\mathcal{S}}(i,i,:)\neq0\}.$$
 An alternative definition of the tensor tubal-rank is that it is
the largest rank of all the frontal slices of $\mathcal{A}\times_3\mathbf F_{n_3}$ in Fourier domain.
\end{mydef}
Suppose the tensor $\mathbf{\mathcal{A}}\in\mathbb{R}^{n_1\times n_2\times n_3}$ has tensor tubal-rank $r$, then the reduced t-SVD of $\mathcal{A}$ is given by
$\mathcal{A} = \mathbf{\mathcal{U}}*\mathbf{\mathcal{S}}*\mathbf{\mathcal{V}}^\text{\rm H},
$
where $\mathbf{\mathcal{U}}\in \mathbb{R}^{n_{1}\times r\times n_{3}}$ and $\mathbf{\mathcal{V}}\in \mathbb{R}^{r\times n_2\times n_{3}}$ are orthogonal tensors, and $\mathbf{\mathcal{S}}\in \mathbb{R}^{r\times r\times n_{3}}$ is an f-diagonal tensor.
An important property of the t-SVD is that the truncated t-SVD of a tensor provides the optimal approximation measured by the Frobenius norm with the tubal
rank at most $r$ \cite{kilmer2013third}.

\begin{mydef}[tensor multi-rank \cite{zhang2014novel}]\label{def:mrank}
Let $\mathcal{A}\in\mathbb{R}^{n_1\times n_2\times n_3}$ be a third-order tensor, the tensor multi-rank, denoted as $\text{rank}_{\text{m}}(\mathcal{A}) \in\mathbb{R}^{n_3}$, is a vector whose $i$-th element is the rank of the $i$-th frontal slice of $\mathcal{B}=\mathcal{A}\times_3\mathbf F_{n_3}$. We can write
\begin{equation}
\text{rank}_{\text{m}}(\mathcal{A}) = \left[\text{rank}\left(\mathcal{B}^{(1)}\right),\text{rank}\left(\mathcal{B}^{(2)}\right)\cdots,\text{rank}\left(\mathcal{B}^{(n_3)}\right)\right].
\end{equation}
\end{mydef}
Given a third-order tensor $\mathcal{A}\in\mathbb{R}^{n_1\times n_2\times n_3}$ we can find that its tensor tubal-rank $\text{rank}_t(\mathcal{A})$ equals to the $\ell_\infty$ norm (or say the biggest value) of the tensor multi-rank $\text{rank}_m(\mathcal{A})$. As $\|\text{rank}_m(\mathcal{X})\|_1\geq\|\text{rank}_m(\mathcal{X})\|_\infty = \text{rank}_t(\mathcal{X}) \geq\frac{1}{n_3}\|\text{rank}_m(\mathcal{X})\|_1$, the tensor tubal-rank is bounded by the $\ell_1$ norm of the tensor multi-rank.}
\begin{mydef}[block diagonal operation \cite{zhang2014novel}]\label{Def:bldg}
The block diagonal operation of $\mathbf{\mathcal{A}}\in \mathbb{C}^{n_{1}\times n_2\times n_{3}}$ is given by
\begin{equation}
\begin{aligned}
 {\tt bdiag}(\mathcal{A})\triangleq\left [
\begin{tabular}{cccc}
$\mathbf{\mathcal{A}}^{(1)}$&&&\\
&$\mathbf{\mathcal{A}}^{(2)}$ &&\\
&&$\ddots$ &\\
&&&$\mathbf{\mathcal{A}}^{(n_3)}$
\end{tabular}\right],
\end{aligned}
\end{equation}
where, and ${\tt bdiag}(\mathcal{A})\in\mathbb{C}^{n_1n_3\times n_2n_3}$.%
\end{mydef}

{\red
\begin{mydef}[tensor-nuclear-norm (TNN) \cite{zhang2014novel}]\label{Def-tnn}
The tensor nuclear norm of a tensor $\mathbf{\mathcal{A}}\in \mathbb{R}^{n_{1}\times n_2\times n_{3}}$, denoted as $\|\mathbf{\mathcal{A}}\|_{\text{\rm TNN}}$, is defined as
\begin{equation}
\begin{aligned}
\|\mathbf{\mathcal{A}}\|_{\text{TNN}}\triangleq\|{\tt bdiag}(\mathcal{A}\times \mathbf F_{n_3})\|_{*}.
\end{aligned}
\label{tnn}
\end{equation}
The TNN can be computed via the summation of the matrix nuclear norm of $(\mathcal{A}\times \mathbf F_{n_3})$'s frontal slices. That is
$\|\mathbf{\mathcal{A}}\|_{\text{TNN}}=\sum_{i=1}^{n_3}\|(\mathcal{A}\times \mathbf F_{n_3})^{(i)}\|_*$.\end{mydef}
Also, as the block circular matrix can be block diagnosed by the DFT, we have
\begin{equation*}
\begin{aligned}
\|\mathcal{A}\|_\text{TNN} = &\|{\tt bdiag}(\mathcal{A}\times \mathbf F_{n_3})\|_{*}\\
& \left\|\left(
\begin{matrix}
  \mathcal{A}^{(1)} & \mathcal{A}^{(n_3)} & \cdots & \mathcal{A}^{(2)} \\
  \mathcal{A}^{(2)} & \mathcal{A}^{(1)} & \cdots & \mathcal{A}^{(3)} \\
  \vdots &  \vdots & \ddots &  \vdots \\
  \mathcal{A}^{(n_3)} & \mathcal{A}^{(n_3-1)} & \cdots & \mathcal{A}^{(1)}
\end{matrix}
\right)\right\|_{*}.
\end{aligned}
\end{equation*}
This reveals the connection between the TNN and the circular convolution operation used for defining the t-prod.

Denoting the Fourier transformed tensor of $\mathcal{X}$ as $\mathcal{Z}=\mathcal{X}\times \mathbf F_{n_3}$, we have $\mathcal{X}=\mathcal{Z}\times_3\mathbf F_{n_3}^{-1}$ and the
the equivalent form of the TNN based LRTC model in \eqref{TNN-LRTC} as
\begin{equation}
\min_\mathcal{Z}\ \|{\tt bdiag}(\mathcal{Z})\|_*\quad \text{s.t.} \ \left(\mathcal{Z}\times_3\mathbf F_{n_3}^{-1}\right)_\Omega = \mathcal{O}_\Omega.
\label{eq-tnn}
\end{equation}}

%
%

\section{Main Results}\label{Sec:Model}
\subsection{Proposed Model}\label{Sec:Model1}
{\red As \eqref{eq-tnn} can be comprehended as to find a slice-wisely low-rank coding of $\mathcal{X}$ with a predefined the dictionary $\mathbf{F}_{n_3}^{-1}$. To promote the flexibility, we replace $\mathbf{F}_{n_3}^{-1}$ with a data-adaptive dictionary and
our} tensor learning and coding model is formulated as
\begin{equation}
\begin{aligned}
\min\limits_{\mathcal{Z},\mathbf D} \quad& \|{\tt bdiag}(\mathcal{Z})\|_*,\\
\text{s.t.}\quad& \left(\mathcal{Z}\times_3\mathbf D\right)_\Omega=\mathcal{O}_{\Omega}\\
&\|\mathbf D(:,i)\|_2= 1\ \text{for} \ i=1,2,\cdots,d,
\end{aligned}
\label{eq-dtnn}
\end{equation}
where $\mathbf D\in\mathbb{R}^{n_3\times d}$ and $\mathcal{Z}\in\mathbb{R}^{n_1\times n_2\times d}$ are respective the dictionary and the low-rank coding coefficients.
As our LRTC model is very similar to the TNN based model in \eqref{eq-tnn}, we term it as dictionary based TNN (DTNN).

{\red One the one hand, if the dictionary $\mathbf D\in\mathbb{R}^{n_3\times d}$ is prefixed and there is a matrix $\mathbf D^*\in\mathbb{R}^{d\times n_3}$ which satisfies $\mathbf D\mathbf D^* = \mathbf I_{n_3}$, the structure of the t-prod (and t-SVD) still holds when replacing $\mathbf F_{n_3}$ and $\mathbf{F}_{n_3}^{-1}$ in \eqref{def-t-prod-f} with $\mathbf D^*$ and $\mathbf D$, respectively. The circular convolution operation between tubes, which is used to define the original t-prod, will change according to how $\mathbf D$ and $\mathbf D^*$ are constructed. If $d=n_3$ and $\mathbf D\mathbf D^* = \mathbf D\mathbf D^* = \mathbf I_{n_3}$, the exact recovery of the underlying tensor is the theoretical guaranteed under certain conditions \cite{lu2019low}.
On the other hand, although the objective function of \eqref{eq-dtnn} is in the same form of \eqref{eq-tnn}, it could not be derived to a normative definition of a norm as Def. \ref{Def-tnn} if finding the $\mathbf D^*$ is difficult. Given a tensor $\mathcal{X}$ and a certain dictionary $\mathbf{D}$, the coefficients in $\mathcal{Z}$ here could not be directly obtained with satisfying $ \left(\mathcal{Z}\times_3\mathbf D\right)=\mathcal{X}$.
It is needed to optimize \eqref{eq-dtnn} simultaneously with respect to the dictionary and coefficients (with $\Omega$ indexing all the entries).
By the way, our DTNN can also be generalized for higher-order tensors via the techniques proposed in \cite{zheng2020tensor,martin2013order}, or for other applications, such as the tensor robust principal component analysis \cite{lu2019tensortpami} and remote sensing images recovery \cite{liu2021hyperspectral,yang2020remote}.}

While resembling \eqref{eq-tnn} in form, our model in \eqref{eq-dtnn} is distinct from the TNN based LRTC model.
The main {\red difference} is that our model is more flexible for different kinds of data because of the data-adaptive dictionary term.
{\red The bottom-right part in Fig. \ref{fig-flowchart} shows the coefficients and the dictionaries obtained by our method.
We can see that the dictionary learned for the HSI completion is smoother than that for the MRI data.
With the adaptively learned dictionary and corresponding low-rank coding, the performance of our method is significantly better than the TNN based LRTC method.}
The dictionary used in \eqref{eq-dtnn} can be viewed as the inverse transform. This is also different from previous works tailoring the linear or unitary transform \cite{lu2019low,song2019robust}.

{Traditional dictionary learning techniques utilize overcomplete dictionaries, the amount of whose atoms is always more than the dimension of the signal, and find the sparse representations \cite{donoho2006most}. In \eqref{eq-dtnn}, although $d$ is much bigger than $n_3$, $\mathbf D$ is still not big enough to overcompletely represent $\mathcal X$, which is of big volume, with sparse coefficients. Therefore, we need the specific low-rank structure of the coefficients, which allows the linear combination of features, together with the learned dictionary, to accurately complete $\mathcal X$.
Thus, our method is distinct from previous tensor dictionary learning methods, which enforce the sparsity of coefficients, e.g., \cite{peng2014decomposable}. Please see Sec. \ref{Sec-SvsLR} for detailed comparisons of sparsity and low-rankness.}
\subsection{Proposed Algorithm}\label{Sec:Alg}
To optimize the specific structured problem in the proposed model, we tailored a multi-block proximal alternating minimization algorithm.
Let
\begin{equation*}\centering
\Phi(\mathbf{\mathcal{X}})=\left\{
\begin{aligned}
&0,\quad       & \mathcal{X}_\Omega = \mathcal{O}_\Omega,\\
&\infty, &\text{otherwise},
\end{aligned}
\right.
\end{equation*}
and
\begin{equation*}\centering
\Psi(\mathbf D)=\left\{
\begin{aligned}
&0,\quad       & \|\mathbf D(:,i)\|_2= 1\ &\text{for} \ i=1,2,\cdots,d,\\
&\infty, &\text{otherwise}&.
\end{aligned}
\right.
\end{equation*}
Thus, the problem in (\ref{eq-dtnn}) can be rewritten as the following unconstraint problem
\begin{equation}
\min\limits_{\mathcal{Z},\mathbf D} \quad\Phi\left(
\mathcal{Z}\times_3\mathbf D\right)+\sum\limits_{{\red i}=1}^{d}\|\mathcal{Z}^{({\red i})}\|_{*}+\Psi(\mathbf D)
\label{A_un}
\end{equation}

As the minimization problem in \eqref{A_un} is difficult to be directly optimized. Therefore, we resort to the half quadratic splitting (HQS) technique \cite{geman1992constrained,nikolova2005analysis} and turn to solve the following problem
\begin{equation}
\min\limits_{\mathcal{Z},\mathbf D, \mathcal{X}}  \frac{\beta}{2}\|\mathcal{X}\hspace{-0.4mm}-\hspace{-0.4mm}
\mathcal{Z}\hspace{-0.4mm}\times_3\hspace{-0.4mm}\mathbf D\|_F^2
\hspace{-0.4mm}+\hspace{-0.4mm}\Phi\left(\mathcal{X}\right)\hspace{-0.4mm}+\hspace{-0.4mm}\sum\limits^{d}_{\red i=1}\|\mathcal{Z}^{({\red i})}\|_{*}\hspace{-0.4mm}+\hspace{-0.4mm}\Psi(\mathbf D).
\label{Obj}
\end{equation}
We denote the objective function in \eqref{Obj} as $L(\mathcal{Z},\mathbf D, \mathcal{X})$.
{\red The optimization problem in \eqref{Obj} is non-convex and has more than two blocks. Thus, it prevent us from directly using some classical algorithms designed for convex optimizations, such as the alternating direction method of multipliers (ADMM) \cite{boyd2011distributed} utilized in \cite{zhang2017exact}, with theoretical convergence guarantees. We employ the proximal alternating minimization framework \cite{PAM2014} for this nonconvex problem with guaranteed convergence. In our algorithm, each variable is alternatively updated as:}
\begin{equation}
\begin{aligned}
  \mathcal{Z}_{k+1} &\hspace{-0.4mm}\in\hspace{-0.4mm} \arg\hspace{-0.4mm}\min\limits_\mathcal{Z}\hspace{-0.7mm}\left\{  \hspace{-0.4mm}L(\mathcal{Z},\mathbf D_{k}, \mathcal{X}_{k}) \hspace{-0.7mm}+\hspace{-0.7mm}\frac{\rho^z_k}{2}\|\mathcal{Z}-\mathcal{Z}_k\|_F^2 \hspace{-0.5mm}\right\}\hspace{-0.5mm}, \\
  \mathbf{D}_{k+1} &\hspace{-0.4mm}\in\hspace{-0.4mm} \arg\hspace{-0.4mm}\min\limits_\mathbf{D}\hspace{-0.7mm}\left\{  \hspace{-0.4mm}L(\mathcal{Z}_{k+1},\mathbf{D}, \mathcal{X}_{k}) \hspace{-0.7mm}+\hspace{-0.7mm}\frac{\rho^d_k}{2}\|\mathbf{D}-\mathbf{D}_{k}\|_F^2 \hspace{-0.5mm}\right\}\hspace{-0.5mm}, \\
  \mathcal{X}_{k+1} &\hspace{-0.4mm}\in\hspace{-0.4mm} \arg\hspace{-0.4mm}\min\limits_\mathcal{X}\hspace{-0.7mm}\left\{\hspace{-0.4mm}L(\mathcal{Z}_{k+1},\mathbf D_{k+1}, \mathcal{X}) \hspace{-0.7mm}+\hspace{-0.7mm}\frac{\rho^x_k}{2}\hspace{-0.4mm}\|\mathcal{X}\hspace{-0.5mm}-\hspace{-0.5mm}\mathcal{X}_k\|_F^2 \hspace{-0.5mm}\right\}\hspace{-0.5mm}, \\
\end{aligned}
\label{updating}
\end{equation}
where $\left(\rho^z_k\right)_{k\in\mathbb{N}}$, $\left(\rho^d_k\right)_{k\in\mathbb{N}}$, and $\left(\rho^x_k\right)_{k\in\mathbb{N}}$ are three positive sequences{\red, and
$\mathcal{Z}_k$, $\mathbf D_k$, and $\mathcal{Z}_k$ respectively denote the values of $\mathcal{Z}$, $\mathbf D$, and $\mathcal{Z}$ at the $k$-th iteration. Thus, for example, $L(\mathcal{Z},\mathbf D_{k}, \mathcal{X}_{k})$ is a function of $\mathcal{Z}$, which comes from $L(\mathcal{Z},\mathbf D, \mathcal{X})$ by fixing other two variables $\mathcal{X}$ and $\mathbf D$ as  $\mathcal{X}_{k}$  and $\mathbf D_{k}$, respectively.}\\

\subsubsection{Updating $\mathcal{Z}$ and $\mathbf D$}
Following the updating strategy in \cite{bao2015dictionary}, the coefficient $\mathcal{Z}$ (or equivalently denoted as $\mathbf Z_{(3)}$ for simplification) and the dictionary $\mathbf D$ {\red at the $k$-th iteration} can be respectively decomposed as followings:
\begin{equation}\red
\begin{aligned}\label{zdecompose}
{\mathbf{Z}_k}_{(3)} \hspace{-1mm}=\hspace{-1mm} \begin{bmatrix}
		{\mathbf z^1_k}^\top\\
\vdots\\
		{\mathbf z^i_k}^\top\\
        \vdots\\
		{\mathbf z^d_k}^\top\\
	\end{bmatrix}\hspace{-1mm}=\hspace{-1mm}
\begin{bmatrix}
		{\tt vec}(\mathbf Z^1_k)^\top\\
\vdots\\
		{\tt vec}(\mathbf Z^i_k)^\top\\
        \vdots\\
	{\tt vec}(\mathbf Z^d_k)^\top\\
	\end{bmatrix}\hspace{-1mm}=\hspace{-1mm}
\begin{bmatrix}
		{\tt vec}(\mathcal{Z}_k(:,:,1))^\top  \\
\vdots\\
		{\tt vec}(\mathcal{Z}_k(:,:,i))^\top  \\
        \vdots\\
		{\tt vec}(\mathcal{Z}_k(:,:,d))^\top  \\
	\end{bmatrix}
  \end{aligned}
\end{equation}%
and
\begin{equation}\label{ddecompose}
\begin{aligned}
\mathbf{D}_k        =  & \left[\mathbf d^1_{k},\cdots,\mathbf d^{i}_{k},\cdots,\mathbf d^{d}_{k}\right],\\
  \end{aligned}
\end{equation}
where {\red $\mathbf{Z}^i_k = \mathcal{Z}_k(:,:,i)$ indicates the $i$-th frontal slice of the coefficients tensor $\mathcal{Z}$ at the $k$-th iteration, $\mathbf z^i_k = {\tt vec}(\mathbf Z^i_k)$,} {\red ${\tt vec}(\cdot)$ denotes the vectorization operation}, and {\red $\mathbf d^{i}_k = \mathbf D_k(:,i)$ is the $i$-th atom of $\mathbf D_k$.}
{\red In our algorithm, the frontal slices of $\mathcal{Z}$ are frequently reshaped into vectors and  vice versa. Therefore, we use ${\tt vec}(\cdot)$ to denote the vectorization from the frontal slices of $\mathcal{Z}$ to a column vector by stacking the columns of $\mathcal{Z}$, and ${\tt vec}(\cdot)^{-1}$ to denote the inverse operation.}

Thus, the $\mathcal{Z}$ subproblem and the $\mathbf D$ subproblem can be respectively {\red split} into $d$ problems. Then, we update the pair of {\red$\mathbf Z^i_{k+1}$} and {\red$\mathbf d^{i}_{k+1}$} from $i = 1$ to $d$. This updating scheme is the same as the well-known KSVD technique \cite{aharon2006k}.
{\red From the decompositions in Eqs. \eqref{zdecompose} and \eqref{ddecompose}, at the beginning of the $k$-th iteartion, we can rewrite the first term in the objective function as $\frac{\beta}{2}\|\mathcal{X}_k-
\mathcal{Z}_k\times_3\mathbf D_k\|_F^2= \frac{\beta}{2}\|{\mathbf X_k}_{(3)}-\mathbf{D}_k {\mathbf{Z}_k}_{(3)}\|_F^2= \frac{\beta}{2}\|{\mathbf X_k}_{(3)}-\sum_{i=1}^d \mathbf d^{i}_{k}{\mathbf z^i_k}^\top\|_F^2$. Thus, for simplicity, we introduce an intermediate variable as
\begin{equation}\mathbf R^{i}_{k} =  {\mathbf X_k}_{(3)}-\sum_{j=1}^{i-1} \mathbf d^{j}_{k+1}(\mathbf z^j_{k+1})^\top-\sum_{j=i+1}^d \mathbf d^{j}_{k}(\mathbf z^j_k)^\top \end{equation}}
{\red Then, we solve following problems:
\begin{equation}
\label{updatingz0}
\begin{aligned}
\mathbf Z^i_{k+1}=\argmin\limits_{\mathbf{Z}}\quad&\frac{\beta}{2}\|\mathbf{R}_{k}^{i}-\mathbf d^{i}_{k}{{\tt vec}(\mathbf Z)}^\top\|_F^2
  +\|\mathbf Z\|_{*} \\
  &+\frac{\rho^z_k}{2}\|\mathbf Z-\mathbf Z_k^i\|_F^2,
  \end{aligned}
\end{equation}
and
\begin{equation}
\label{updatingd}
\begin{aligned}
\mathbf d^i_{k+1} = \arg\min\limits_{\mathbf{d}}\quad&\frac{\beta}{2}\|\mathbf{R}_{k}^{i}-\mathbf d(\mathbf z^i_k)^\top\|_F^2+\Psi(\mathbf d) \\&+\frac{\rho^d_k}{2}\|\mathbf d-\mathbf d_k^i\|_F^2.
  \end{aligned}
\end{equation}


After denoting $\mathbf{z}={\tt vec}(\mathbf Z)$, two quadratic terms in \eqref{updatingz0} can be combined as
\begin{equation*}\label{deriveZ}
\begin{aligned}
\frac{\beta}{2}&\|\mathbf{R}_{k}^{i}-\mathbf d^{i}_{k}\mathbf{z}^\top\|_F^2+\frac{\rho^z_k}{2}\|\mathbf Z-\mathbf Z_k^i\|_F^2\\
=&\frac{\beta}{2}\left(\langle\mathbf{R}_{k}^{i},\mathbf{R}_{k}^{i}\rangle -2\langle\mathbf{R}_{k}^{i},\mathbf d^{i}_{k}\mathbf{z}^\top\rangle +\langle\mathbf d^{i}_{k}\mathbf{z}^\top,\mathbf d^{i}_{k}\mathbf{z}^\top\rangle\right)\\
&+\frac{\rho^z_k}{2}\left(\langle\mathbf Z,\mathbf Z\rangle-2\langle\mathbf Z,\mathbf Z_k^i\rangle+\langle\mathbf Z_k^i,\mathbf Z_k^i\rangle\right)\\
=&\frac{\beta}{2} \left(\langle\mathbf{R}_{k}^{i},\mathbf{R}_{k}^{i}\rangle -2\langle{\tt vec}^{-1}((\mathbf{R}_{k}^{i})^\top\mathbf d^{i}_{k}),\mathbf{z}\rangle+\langle\mathbf{z},\mathbf{z}\rangle\right)\\
&+\frac{\rho^z_k}{2}\left(\langle\mathbf Z,\mathbf Z\rangle-2\langle\mathbf Z,\mathbf Z_k^i\rangle+\langle\mathbf Z_k^i,\mathbf Z_k^i\rangle\right)\\
=&\frac{\beta+\rho_k^z}{2}\|\mathbf Z-\frac{\rho_k^z\mathbf Z_k^i+\beta{\tt vec}^{-1}((\mathbf{R}_{k}^{i})^\top\mathbf d^{i}_{k})}{\beta+\rho_k^z}\|_F^2+\frac{\rho^z_k}{2}\|\mathbf Z_k^i\|_F^2\\
&-\frac{1}{2(\beta+\rho_k^z)}\|\rho_k^z\mathbf Z_k^i+\beta{\tt vec}^{-1}((\mathbf{R}_{k}^{i})^\top\mathbf d^{i}_{k})\|_F^2+\frac{\beta}{2}\|\mathbf{R}_{k}^{i}\|_F^2.
  \end{aligned}
\end{equation*}
Therefore, leaving terms independent of $\mathbf Z$ and adding the nuclear norm term, the minimization problem in \eqref{updatingz0} is equivalent to:
\begin{equation}
\label{updatingz2}
\begin{aligned}
 \mathbf{Z}_{k+1}^i \in \arg\min\limits_\mathbf{Z}\quad\|\mathbf Z\|_{*}+\frac{\beta+\rho_k^z}{2}\|\mathbf Z-\mathbf M^i_k\|_F^2,
  \end{aligned}
\end{equation}
where $\mathbf M^i_k = \frac{\rho_k^z\mathbf Z_k^i+\beta{\tt vec}^{-1}({(\mathbf{R}_{k}^{i})}^\top\mathbf d^{i}_{k})}{\beta+\rho_k^z}$.
Then, we can directly derive the closed form solution of \eqref{updatingz2} with the singular value thresholding (SVT) operator \cite{cai2010singular} as
\begin{equation}
\label{updatingz}
\begin{aligned}
  \mathbf{Z}_{k+1}^i={\tt SVT}_\frac{1}{\beta+\rho_k^z} \left(\mathbf M^{i}_k\right)\triangleq\mathbf U\left(\mathbf S-\frac{1}{\beta+\rho_k^z}\right)_+\mathbf V^\top,
  \end{aligned}
\end{equation}
where $(\mathbf U,\mathbf S,\mathbf V)$ comes from the SVD of $\mathbf M^i_k$, $\mathbf S$ is a diagonal matrix with $\mathbf M^{k,i}$'s singular values, and $(\cdot)_+$ means keeping the positive values and setting the negative values as 0.

%

Similarly, we can obtain the closed form solution of \eqref{updatingd} as following:
\begin{equation}\label{D12}
\begin{aligned}
\mathbf{d}^i_{k+1} = \frac{\beta\mathbf R^{i}_k{\tt vec}(\mathbf{Z}_{k}^i)+\rho_k^d\mathbf d_{i}^k}{\|\beta\mathbf R^{i}_k{\tt vec}(\mathbf{Z}_{k}^i)+\rho_k^d\mathbf d_{i}^k\|_2}.
  \end{aligned}
\end{equation}}
Afterwards, we obtain $\mathcal{Z}_{k+1}$ with its $i$-th frontal slice equaling to $\mathbf{Z}_{k+1}^i$ and $\mathbf{D}_{k+1} = \left[\mathbf d^1_{k+1},\cdots,\mathbf d^{i}_{k+1},\cdots,\mathbf d^{d}_{k+1}\right]$.\\

\subsubsection{Updating $\mathcal{X}$} We update $\mathcal{X}$ via solving the following minimization problem:
\begin{equation*}
\label{updatingX}
\begin{aligned}
\min\limits_\mathcal{X}\frac{\beta}{2}\|\mathcal{X}-\mathcal{Z}_{k+1}\times_3\mathbf D_{k+1}\|_F^2+\Phi\left(\mathcal{X}\right) +\frac{\rho^x_k}{2}\|\mathcal{X}-\mathcal{X}_k\|_F^2.
  \end{aligned}
\end{equation*}
$\mathcal{X}_{k+1}$ is updated via the following steps:
\begin{equation}
\label{X1}
\left\{
\begin{aligned}
  \mathcal{X}_{k+\frac{1}{2}} =&\frac{\beta{\tt fold}_3(\mathbf{D}_{k+1}{\mathbf{Z}_{k+1}}_{(3)})+\rho^x_k\mathcal{X}_k}{\beta+ \rho^x_k},  \\
  \mathcal{X}_{k+1} =&\left( \mathcal{X}_{k+\frac{1}{2}}\right)_{\Omega^C}+\mathcal{O}_{\Omega}.
  \end{aligned}\right.
\end{equation}
{\red where $\Omega_C$ denotes the complementary set of the $\Omega$. }
Finally, the pseudocode  is summarized in Algorithm \ref{alg}. The computation complexity of our algorithm at each iteration is O$(dn_1n_2(dn_3+\min(n_1,n_2)+n_3)$, given an input with size $n_1\times n_2\times n_3$.

\begin{algorithm}[htp]
\renewcommand\arraystretch{0.8}
\caption[Caption for LOF]{Proximal alternating minimization algorithm for solving \eqref{Obj}}
\begin{algorithmic}[1]
\renewcommand{\algorithmicrequire}{\textbf{Input:}} 
\Require
The observed tensor $\mathcal{O}\in\mathbb{R}^{n_1\times n_2\times n_3}$; the set of observed entries $\Omega$.
\renewcommand{\algorithmicrequire}{\textbf{Initialization:}} 
\Require $\mathcal{X}^{(0)}$, $\mathbf D^0$, and $\mathbf Z^0$;
\While {not converged}
\For {$i$ = 1 to $d$}
\State Update $\mathcal{Z}^{k+1}(:,:,i)$ via Eq. \eqref{updatingz2};
\State Update $\mathbf D^{k+1}(:,i)$ via \eqref{D12};
\EndFor
\State Update $\mathcal{X}^k$  (\ref{X1}).
\EndWhile
\renewcommand{\algorithmicrequire}{\textbf{Output:}}
\Require The reconstructed tensor $\mathcal{X}$.
\end{algorithmic}
\label{alg}
\end{algorithm}

\subsection{Convergency analysis}
In this part, we are really to establish the theoretical guarantee of convergence on our algorithm.
{\red
For convenience, we first define the following formularies,
\begin{equation}
\nonumber{}
\begin{aligned}
F(\mathcal{Z}) =\ &\sum\limits_{k=1}^{d}\|\mathcal{Z}^{(k)}\|_{*} = \|{\tt bdiag}(\mathcal{Z})\|_*,\\
\delta_{\mathcal{X}}(\mathcal{X})=\ &\Phi\left(\mathcal{X}\right),\\
\delta_{\mathcal{D}}(\mathbf D)=\ &\Psi(\mathbf D),\\
Q(\mathcal{Z},\mathbf D,\mathcal{X})=\ &\frac{\beta}{2}\|\mathcal{X}-
\mathcal{Z}\times_3\mathbf D\|_F^2,\\
L(\mathcal{Z},\mathbf D, \mathcal{X}) =\ &F(\mathcal{Z}) + \delta_{\mathcal{X}}(\mathcal{X})+\delta_{\mathcal{D}}(\mathbf D)+Q(\mathcal{Z},\mathbf D,\mathcal{X}),
\end{aligned}
\end{equation}
and
\begin{equation}\left\{
\begin{aligned}\label{updatingDZX}
  \mathcal{Z}_{k+1}= &\arg\min\limits_\mathcal{Z}\{  M_1 (\mathcal{Z}|{\red\mathcal{Z}_k}):=F(\mathcal{Z})
  \\& +Q(\mathcal{Z},\mathbf D_{k}, \mathcal{X}_{k}) +\frac{\rho^z_k}{2}\|\mathcal{Z}-\mathcal{Z}_k\|_F^2 \}, \\
  \mathbf{D}_{k+1}= &\arg\min\limits_\mathbf{D}\{ M_2 (\mathbf D|{\red\mathbf{D}_k)}:= \delta_{\mathcal{X}}(\mathcal{X})
  \\& +Q(\mathcal{Z}_{k+1},\mathbf{D}, \mathcal{X}_{k}) +\frac{\rho^d_k}{2}\|\mathbf{D}-\mathbf{D}_{k}\|_F^2 \}, \\
  \mathcal{X}_{k+1}= &\arg\min\limits_\mathcal{X}\{ M_3 (\mathcal{X}|{\red\mathcal{X}_k}):= \delta_{\mathcal{D}}(\mathbf D)
  \\&+Q(\mathcal{Z}_{k+1},\mathbf D_{k+1},\mathcal{X}) +\frac{\rho^x_k}{2}\|\mathcal{X}-\mathcal{X}_k\|_F^2 \}. \\
\end{aligned}\right.
\end{equation}
Next, we give the theorem of the global convergency of the sequence generated by \eqref{updatingDZX} as follows.
\begin{thm}\label{th-1}
The sequence generated by \eqref{updatingDZX} is bounded, and it converges to a critical point of $L(\mathcal{Z},\mathbf D,\mathcal{X})$.
\end{thm}

As the process of updating in \eqref{updatingDZX} is factually a special instance of the algorithm 4 described in \cite{PAMsequrence},
the proof of Theorem \ref{th-1} confirms to Theorem 6.2 of \cite{PAMsequrence} if satisfying the following conditions:
$$
\left\{
\begin{aligned}
\text{i)}&\ \text{the K-\L property of $L$ at each point,}\\
\text{ii)}&\ \text{the sufficient decrease condition ((64) in \cite{PAMsequrence}),}\\
\text{iii)}&\ \text{the relative error condition ((65)-(66) in \cite{PAMsequrence}).}
\end{aligned}
\right.
$$
The road map of the proof also follows this line.
Before verifying these conditions, we first give some basic definitions from variational analysis \cite{rockafellar2009variational,clarke2008nonsmooth}.
If $f:\mathbb{R}^n\rightarrow\mathbb{R}\cup \{+\infty\}$ is a real-extended-valued function, its domain is given by $\text{dom}f:=\{x\in\mathbb{R}^{n}:f(x)<+\infty \}$.
For each $x\in\text{dom}f$, the {\it Fr\'{e}chet subdifferential} of $f$ at $x$, written $\hat{\partial}f(x)$, is the set of vectors $x^*\in\mathbb{R}^n$ which satisfy
$$
\liminf\limits_{y\neq x,y\rightarrow x}\frac{1}{\|x-y\|}[f(y)-f(x)-\langle x^*,y-x\rangle\geq0.
$$
When $x\notin \text{dom}f$, we set $\hat{\partial}f(x)=\emptyset$.
Then, the subdifferential ({\it limiting-subdifferential} \cite{rockafellar2009variational}) of $f$ at $x\in\text{dom}f$, written $\partial f(x)$, is a set defined as
$$
\left\{x^*\in\mathbb{R}^n:\exists x_n\rightarrow x, f(x_n)\rightarrow f(x), x_n^*\in\hat{\partial}f(x_n)\rightarrow x^*\right\}.
$$
The (limiting-) subdifferential is more stable than the {Fr\'{e}chet subdifferential} in an algorithmic context which involves limiting processes.
A necessary (but not sufficient) condition for $x\in\mathbb{R}^n$ to be a minimizer of $f$ is
$\partial f(x)\ni 0$.
A point that satisfies $\partial f(x)\ni 0$ is called limiting-critical or simply critical.
If $K$ is a subset of $\mathbb{R}^n$ and $x$ is any point in $\mathbb{R}^n$, we set
$$
\text{dist}(x,K) = \inf\{\|x-z\|:z\in K\}.
$$
If $K$ is empty, we have $\text{dist}(x,K)=+\infty$ for all $x\in\mathbb{R}^n$. For any real-extended-valued function $f$ on $\mathbb{R}^n$, we have $\text{dist}(0,\partial f(x)) = \inf\{\|x^*\|:x^*\in\partial f(x)\}$ \cite{PAM}.
Let $f:\mathbb{R}^n\rightarrow\mathbb{R}\cup\{+\infty\}$ be a proper lower semicontinuous function. For $-\infty<\eta_1<\eta_2\leq+\infty$, we set
$$
[\eta_1<f<\eta_2] = \{x\in\mathbb{R}^n:\eta_1<f(x)<\eta_2\}.
$$
Then, we can define K-\L\ functions and semi-algebraic functions.}
\begin{mydef} (Kurdyka-\L ojasiewicz property \cite{PAMsequrence})
A proper lower semi-continuous function $f: \mathbb{R}^{n} \rightarrow \mathbb{R} \cup\{+\infty\}$ is
said to have the K-\L\ property at $\bar{x} \in \textrm{dom}(\partial f)$ if there exist $\eta \in (0,+\infty]$, a neighborhood $U$ of $\bar{x}$ and a continuous concave function $\phi:[0,\eta)\rightarrow[0,+\infty]$, which satisfies $\phi(0)=0$, $\phi$ is $C^1$ on $(0,\eta)$, and $\phi(s)>0, \forall s\in (0,\eta))$ such that for each $x \in U\cap[f(\bar{x})<f<f(\bar{x})+\eta]$  the K-\L \ inequality holds:
\begin{equation}
\begin{aligned}
\phi'(f(x)-f(\bar{x}))\text{dist}(0,\partial f(x))\geq1. 
  \end{aligned}
\end{equation}
If $f$ satisfies the K-\L \ property at each point of $\textrm{dom}\partial f$ then $f$ is called a K-\L \ function.
\end{mydef}
\begin{mydef}(Semi-algebraic sets and functions \cite{PAMsequrence}) A subset $S$ of $\mathbb{R}$ is called the semi-algebraic set if there
exists a finite number of real polynomial functions $g_{ij}, h_{ij}$ such that $S=\bigcap_{j}\bigcup_{i}\{x\in \mathbb{R}^n: g_{ij}(x)=0,h_{ij}(x)<0\}$. A function $f$ is called the semi-algebraic function if its graph $\{(x,t)\in\mathbb{R}^n\times\mathbb{R}, t=f(x)\}$ is a semi-algebraic set.
\end{mydef}
{\red
Next, we verify the K-\L\  property of $L$ and then show the descent Lemma for $L(\mathcal{Z},\mathbf D_{k}, \mathcal{X}_{k})$. Afterwards, the relative error Lemma would be given. Finally, we establish the proof of Theorem \ref{th-1}.
\begin{lemma}[K-\L\  property Lemma]\label{semitoKL}
The function $L$
satisfies the K-\L\  property at each point.
\end{lemma}
\begin{proof}[Proof of Lemma \ref{semitoKL}]
It is easy to verify that $Q$ is $\emph{C}^1$ function with locally Lipschitz continuous gradient and $F$, $\delta_{D}$, and $\delta_{\mathcal{X}}$ are proper and lower semi-continuous. Thus, $L$ is a proper lower semi-continuous function.
The nuclear norm and Frobenius norm are semialgebraic \cite{PAM2014}. Additionally, the indicator function with semialgebraic sets is semialgebraic \cite{PAM2014}.
As a semi-algebraic real valued function $f$ is a K-\L\ function, i.e., $f$ satisfies K-\L \ property at each $x \in \textrm{dom}(f)$ \cite{semitoKL},
the function $L$
satisfies the K-\L property at each point.
\end{proof}}

\begin{lemma}[Descent Lemma]\label{descent}
Assume that $L(\mathcal{Z},\mathbf D, \mathcal{X})$ is a $C^1$ function with locally Lipschitz continuous gradient and $\rho^z_k, \rho^d_k, \rho^x_k >0$. Let
$\{\mathcal{Z}_k,\mathbf{D}_{k},\mathcal{X}_{k}\}_{k \in \mathbb{N}}$ is generated by \eqref{updatingDZX}. Then
\begin{equation*}
\begin{aligned}
F(\mathcal{Z}_{k+1})+Q(\mathcal{Z}_{k+1},\mathbf D_{k}, \mathcal{X}_{k}) +\frac{\rho^z_k}{2}\|\mathcal{Z}_{k+1}-\mathcal{Z}_k\|_F^2 \leq &\\ F(\mathcal{Z}_{k})+Q(\mathcal{Z}_{k},\mathbf D_{k}, \mathcal{X}_{k}),&\\
\delta_{\mathcal{D}}(\mathbf D_{k+1})+Q(\mathcal{Z}_{k+1},\mathbf{D}_{k+1}, \mathcal{X}_{k}) +\frac{\rho^d_k}{2}\|\mathbf{D}_{k+1}-\mathbf{D}_{k}\|_F^2 \leq& \\ \delta_{\mathcal{D}}(\mathbf D_{k})+Q(\mathcal{Z}_{k+1},\mathbf{D}_{k}, \mathcal{X}_{k}),&\\
\delta_{\mathcal{X}}(\mathcal{X}_{k+1})+Q(\mathcal{Z}_{k+1},\mathbf D_{k+1}, \mathcal{X}_{k+1}) +\frac{\rho^x_k}{2}\|\mathcal{X}_{k+1}-\mathcal{X}_k\|_F^2\leq &\\ \delta_{\mathcal{X}}(\mathcal{X}_{k})+Q(\mathcal{Z}_{k+1},\mathbf D_{k+1}, \mathcal{X}_{k}).&
\end{aligned}
\end{equation*}
\end{lemma}

\begin{proof}[Proof of Lemma \ref{descent}]
When $\mathbf{D}_{k+1}$ and $\mathcal{X}_{k+1}$ are optimal solutions of $M_2$ and $M_3$, $\delta_{\mathcal{D}} = 0$ and $\delta_{\mathcal{X}} = 0$. By the definitions of $M_1, M_2$, and $M_3$, we clearly have that
\begin{equation*}
\begin{aligned}
F(\mathcal{Z}_{k+1})+Q(\mathcal{Z}_{k+1}&,\mathbf D_{k}, \mathcal{X}_{k}) +\frac{\rho^z_k}{2}\|\mathcal{Z}_{k+1}-\mathcal{Z}_k\|_F^2 \\
= M_1(\mathcal{Z}_{k+1}|\mathcal{Z}_{k})\leq & M_1(\mathcal{Z}_{k}|\mathcal{Z}_{k})\\
=&F(\mathcal{Z}_{k})+Q(\mathcal{Z}_{k},\mathbf D_{k}, \mathcal{X}_{k}),\\
\end{aligned}
\end{equation*}
\begin{equation*}
\begin{aligned}
\delta_{\mathcal{D}}(\mathbf D_{k+1})+Q(\mathcal{Z}_{k+1}&,\mathbf{D}_{k+1}, \mathcal{X}_{k}) +\frac{\rho^d_k}{2}\|\mathbf{D}_{k+1}-\mathbf{D}_{k}\|_F^2 \\
= M_2(\mathbf{D}_{k+1}|\mathbf{D}_{k})\leq & M_2(\mathbf{D}_{k}|\mathbf{D}_{k})\\
=&\delta_{\mathcal{D}}(\mathbf D_{k})+Q(\mathcal{Z}_{k+1},\mathbf{D}_{k}, \mathcal{X}_{k}),\\
\end{aligned}
\end{equation*}
\begin{equation*}
\begin{aligned}
\delta_{\mathcal{X}}(\mathcal{X}_{k+1})+Q(\mathcal{Z}_{k+1}&,\mathbf D_{k+1}, \mathcal{X}_{k+1}) +\frac{\rho^x_k}{2}\|\mathcal{X}_{k+1}-\mathcal{X}_k\|_F^2\\
= M_3(\mathcal{X}_{k+1}|\mathcal{X}_{k})\leq & M_3(\mathcal{X}_{k}|\mathcal{X}_{k}) \\
=& \delta_{\mathcal{X}}(\mathcal{X}_{k})+Q(\mathcal{Z}_{k+1},\mathbf D_{k+1}, \mathcal{X}_{k}).
\end{aligned}
\end{equation*}
The descent lemma has been proved.
\end{proof}

\begin{lemma}[Relative error Lemma]\label{rerror}
 $\{\mathcal{Z}_k,\mathbf{D}_{k},\mathcal{X}_{k}\}_{k\in\mathbb{N}}$ is generated by \eqref{updatingDZX} and $\rho^z_k, \rho^d_k, \rho^x_k >0$. Then there exists $V_{1,k+1}, V_{2,k+1}, V_{3,k+1}$, which satisfy the following formularies,
\begin{equation*}
\begin{aligned}
\|V_{k+1}^1\hspace{-1mm}+\nabla_{\mathcal{Z}}Q(\mathcal{Z}_{k+1},\mathbf D_{k}, \mathcal{X}_{k})\|_F &\hspace{-.5mm}\leq\hspace{-.5mm} \rho^z_k\|\mathcal{Z}_{k+1}\hspace{-.5mm}-\hspace{-.5mm}\mathcal{Z}_{k}\|_F,\\
\|V_{k+1}^2\hspace{-1mm}+\hspace{-1mm}\nabla_{\mathbf{D}}Q(\mathcal{Z}_{k+1},\mathbf D_{k+1}, \mathcal{X}_{k})\|_F &\hspace{-.5mm}\leq\hspace{-.5mm} \rho^d_k\|\mathbf{D}_{k+1}\hspace{-.5mm}-\hspace{-.5mm}\mathbf{D}_{k}\|_F,\\
\|V_{k+1}^3\hspace{-1mm}+\hspace{-1mm}\nabla_{\mathcal{X}}Q(\mathcal{Z}_{k+1},\mathbf D_{k+1}, \mathcal{X}_{k+1})\|_F &\hspace{-.5mm}\leq\hspace{-.5mm} \rho^x_k\|\mathcal{X}_{k+1}\hspace{-.5mm}-\hspace{-.5mm}\mathcal{X}_{k}\|_F,\\
\end{aligned}
\end{equation*}
where $V_{k+1}^1 \in \partial F(\mathcal{Z}_{k+1})$, $V_{k+1}^2 \in \partial \delta_{\mathcal{D}}(\mathbf D_{k+1})$, $V_{k+1}^3 \in \partial \delta_{\mathcal{X}}(\mathcal{X}_{k+1})$, and $\nabla$ indicates the (partial) gradient.
\end{lemma}
\begin{proof}[Proof of Lemma \ref{rerror}] 
By the definition of $M_1$, $M_2$, and $M_3$, we have
\begin{equation*}
\begin{aligned}
0 \in &\partial F(\mathcal{Z}_{k+1})\hspace{-.5mm}+\hspace{-.5mm}\nabla_{\mathcal{Z}}Q(\mathcal{Z}_{k+1},\mathbf D_{k}, \mathcal{X}_{k})\hspace{-.5mm}+\hspace{-.5mm}\rho_k^z(\mathcal{Z}_{k+1}\hspace{-.5mm}-\hspace{-.5mm}\mathcal{Z}_{k}),\\
0 \in &\partial \delta_{\mathcal{D}}(\mathbf D_{k+1})\hspace{-.5mm}+\hspace{-.5mm}\nabla_{\mathbf{D}}Q(\mathcal{Z}_{k+1},\mathbf D_{k+1}, \mathcal{X}_{k})
\hspace{-.5mm}+\hspace{-.5mm}\rho_k^d(\mathbf D_{k+1}\hspace{-.5mm}-\hspace{-.5mm}\mathbf D_{k}),\\
0 \in &\partial \delta_{\mathcal{X}}(\mathcal{X}_{k+1})\hspace{-.5mm}+\hspace{-.5mm}\nabla_{\mathcal{X}}Q(\mathcal{Z}_{k+1},\mathbf D_{k+1}, \mathcal{X}_{k+1})
\hspace{-.5mm}+\hspace{-.5mm}\rho_k^x(\mathcal{X}_{k+1}\hspace{-.5mm}-\hspace{-.5mm}\mathcal{X}_{k}).
\end{aligned}
\end{equation*}
Let
\begin{equation*}\left\{
\begin{aligned}
V_{k+1}^1 := &-\nabla_{\mathcal{Z}}Q(\mathcal{Z}_{k+1},\mathbf D_{k}, \mathcal{X}_{k})-\rho_k^z(\mathcal{Z}_{k+1}-\mathcal{Z}_{k})\\
V_{k+1}^2:=&-\nabla_{\mathcal{Z}}Q(\mathcal{Z}_{k+1},\mathbf D_{k+1}, \mathcal{X}_{k})-\rho_k^d(\mathbf D_{k+1}-\mathbf D_{k})\\
V_{k+1}^3:=&-\nabla_{\mathcal{Z}}Q(\mathcal{Z}_{k+1},\mathbf D_{k+1}, \mathcal{X}_{k})-\rho_k^x(\mathcal{X}_{k+1}-\mathcal{X}_{k})\\
\end{aligned}\right.
\end{equation*}
It is clear that $V_{k+1}^1\in \partial F(\mathcal{Z}_{k+1})$, $V_{k+1}^2\in \partial \delta_{\mathcal{D}}(\mathbf D_{k+1})$, and $V_{k+1}^3\in \partial \delta_{\mathcal{X}}(\mathcal{X}_{k+1})$. Thus, we have
\begin{equation*}\left\{
\begin{aligned}
&\|V_{k+1}^1+\nabla_{\mathcal{Z}}Q(\mathcal{Z}_{k+1},\mathbf D_{k}, \mathcal{X}_{k})\|_F = \rho_k^z\|\mathcal{Z}_{k+1}-\mathcal{Z}_{k}\|_F,\\
&\|V_{k+1}^2+\nabla_{\mathbf D}Q(\mathcal{Z}_{k+1},\mathbf D_{k+1}, \mathcal{X}_{k})\|_F  = \rho_k^d\|\mathbf D_{k+1}-\mathbf D_{k}\|_F,\\
&\|V_{k+1}^3\hspace{-.5mm}+\hspace{-.5mm}\nabla_{\mathcal{X}}Q(\mathcal{Z}_{k+1},\mathbf D_{k+1}, \mathcal{X}_{k+1})\|_F  \hspace{-.5mm}=\hspace{-.5mm} \rho_k^x\|\mathcal{X}_{k+1}\hspace{-.5mm}-\hspace{-.5mm}\mathcal{X}_{k}\|_F.\\
\end{aligned}\right.
\end{equation*}
The proof of relative error Lemma has been finished.
\end{proof}

Now, we begin to establish our proof of Theorem \ref{th-1}.

{\red\begin{proof}[Proof of Theorem \ref{th-1}]
From Lemma \ref{descent}, we have that the objective function value monotonically decreases.
Firstly, we can see that the indicator function $\delta_\mathcal{D}(\mathbf D) = \Phi(\mathbf D)$ should be 0 from its definition. Thus, $$\|\mathbf D\|^2_F = \sum_i \|\mathbf D(:,i)\|^2_2  = d,$$ which means {\it $\{\mathbf{D}_k\}_{k\in\mathbb{N}}$ is bounded}.
Meanwhile, from the monotonic decreasing, the nonnegative terms $F(\mathcal{Z}) = \sum_{k=1}^{d} \|\mathcal{Z}^{(k)} \|_{*}$ and $ Q(\mathcal{Z},\mathbf D,\mathcal{X})= \frac{\beta}{2}\|\mathcal{X} - \mathcal{Z} \times_3 \mathbf D\|_F^2$ are bounded.
Then,  \[\|\mathcal{Z}\|_F^2  =   \sum_{k=1}^{d} \|\mathcal{Z}^{(k)}\|_F^2  \leq   \sum_{k=1}^{d} \|(\mathcal{Z}^{(k)})\|_*^2.\]
That is, {\it $\{\mathbf{Z}_k\}_{k\in\mathbb{N}}$ is bounded}. 
Next, from the triangle inequality, we have $$\|\mathcal{X}\|_F - \|\mathcal{Z}\|_F\|\mathbf D\|_F
\leq  \|\mathcal{X}\|_F - \|\mathcal{Z} \times_3 \mathbf D\|_F
 \leq \|\mathcal{X}- \mathcal{Z} \times_3 \mathbf D\|_F.$$ This is equivalent to $$\|\mathcal{X}\|_F \leq   \|\mathcal{X} - \mathcal{Z} \times_3 \mathbf D\|_F + \|\mathcal{Z}\|_F\|\mathbf D\|_F.$$
Therefore,  {\it $\{\mathbf{X}_k\}_{k\in\mathbb{N}}$ is bounded}.

By lemma \ref{semitoKL}, the sequence $\{\mathcal{Z}_k,\mathbf D_k,\mathcal{X}_k\}_{k\in\mathbb{N}}$ is a bounded sequence with the K-\L\ property at each point.
Combining Lemma \ref{descent} and Lemma \ref{rerror} with the above property of $L$, the process of updating in \eqref{updatingDZX} is factually a special instance of the algorithm 4 described in \cite{PAMsequrence}.
Lemma \ref{descent} and Lemma \ref{rerror} correspond to the (64)-(65)-(66) in \cite{PAMsequrence}.
Under these conditions,
this proof conforms to Theorem 6.2 of \cite{PAMsequrence}.
Thus, the bounded sequence $\{\mathcal{Z}_k,\mathbf D_k,\mathcal{X}_k\}_{k\in\mathbb{N}}$ converges to a critical
point of $L(\mathcal{Z},\mathbf D,\mathcal{X})$.
\end{proof}
}
Algorithm \ref{alg} is a direct multi-block generalization of \eqref{updatingDZX}. The proof of its convergence accords with the proof of Theorem \ref{th-1} and can be easily obtained. Meanwhile, the above convergence analysis is more similar to the analysis in \cite{PAMsequrence}, being convenient for the verification of readers. Therefore, we establish the proof of Theorem \ref{th-1} here.

\section{Numerical Experiments}\label{Sec-Exp}

In this section, we compare our method with other state-of-the-art methods.
Compared methods consist of:
one baseline Tucker-rank based method HaLRTC\footnote{\scriptsize\url{https://www.cs.rochester.edu/~jliu/code/TensorCompletion.zip}} \cite{Liu2013PAMItensor},
a Bayesian CP-factorization based method (BCPF\footnote{\scriptsize\url{https://github.com/qbzhao/BCPF}}) \cite{zhao2015bayesian},
a tensor ring decomposition based method (TRLRF\footnote{\scriptsize\url{https://github.com/yuanlonghao/TRLRF}}) \cite{yuan2019tensor},
a t-SVD based method (TNN\footnote{\scriptsize\url{https://github.com/jamiezeminzhang/Tensor_Completion_and_Tensor_RPCA}}) \cite{zhang2017exact},
a DCT induced TNN minimization method (DCTNN\footnote{\scriptsize Implemented by ourselves based on the code of TNN}) \cite{lu2019low},
and a framelet represented TNN minimization method (FTNN\footnote{\scriptsize\url{https://github.com/TaiXiangJiang/Framelet-TNN}}) \cite{jiang2019framelet}.
We select four types of tensor data, including videos, HSIs, traffic data, and MRI data, to show that our method is adaptive to different types of data.

Since the algorithm of our method is a non-convex approach, the initialization of our algorithm is important.
We use a simple linear interpolation strategy, which is employed in \cite{yair2018multi} and convenient to implement with low cost, to fill
in the missing pixels and obtaining $\mathcal{X}_0$ for our method.
{\red As the index of observed entries $\Omega$ is known, we first sort $n_1n_2$ tubes of $\mathcal{X}_0\in\mathbb{R}^{n_1\times n_2\times n_3}$ based on the number of observed entries in each tube. Then, we select the first $d$ tubes, which contain the observed entries as much as possible to construct $\mathbf D\in\mathbb{R}^{d\times n_3}$. Finally, the columns of $\mathbf D$ are nomoalized to satisfy $\|\mathbf D(:,i)\|_2 = 1$ for $i = 1,\cdots,d$. This strategy comes from many traditionary dictionary learning techniques, such as [44]. Then, we fix $\mathcal{X}=\mathcal{X}_0$ and run 10 iterations of our method to initialize the $\mathcal{Z}_0$ with random inputs.} 

{\red
Throughout all the experiments in this paper, parameters of the proposed method are set as: $d=5n_3$, $\beta = 10$, $\rho^z=20$, $\rho^d=1$, and $\rho^x=1$.
In the framework of the HQS algorithm, the penalty parameter $\beta$ is required to reach infinite when iteration goes on. Therefore, we enlarge $\beta$ at the 15-th, 20-th, and 25-th iterations by multiplying the factor 1.5 and enlarge $\beta$ by multiplying the factor 1.2 at each iteration from the 30-th iteration until satisfying the condition of convergence. $\rho^z$, $\rho^d$, and $\rho^x$ are selected by grid search form the candidate set $\{0.1,0.2,0.5,1,2,5,10,20,50,100\}$, while $d$ and $\rho$ are manually tuned.}

{\red As for the compared methods, their parameters are manually tuned for the best performances.  Specifically, as the models of TNN and DCTNN are optimized by ADMM, we set the parameter $\beta$, which is introduced when building the argument Lagrangian function, as $10^{-2}$ at the beginning and enlarge it with a factor $1.2$ at each iteration.
For other methods,
we set
i) $\alpha = [1,1,5]$ and $\rho=10^{-2}$ (as referred to (42) in [4]) for HaLRTC,
ii) removing unnecessary components automatically, random initializations, and the initial rank 200 for  BCPF,
iii) $\text{TR-rank}=12$, $\mu=1$ and $\lambda=10$ (as referred to (15) in [21]) for TRLRF ,
iv) using the default setting in [36] for  FTNN.}

\begin{table*}[!t]
\renewcommand\arraystretch{0.75}\setlength{\tabcolsep}{3pt}\footnotesize
\caption{PSNR, SSIM, and UIQI of results by different methods with different sampling rates on the \textbf{video} data. The {\color{red}best}, the {\color{blue}second best}, and the {\color{greed} third best} values are respectively highlighted by {\color{red}red}, {\color{blue}blue}, and {\color{greed}green} colors.}\vspace{-2mm}
\centering
\begin{tabular}{llccccccccccccccccccccccc}\toprule
 \multirow{3}{*}{Video} & SR    & \multicolumn{3}{c}{10\%} && \multicolumn{3}{c}{20\%} && \multicolumn{3}{c}{30\%} && \multicolumn{3}{c}{40\%} && \multicolumn{3}{c}{50\%} &\multirow{2}{*}{Time} \\\cmidrule(r){2-21}
&
Method & \footnotesize PSNR  &\footnotesize SSIM &\footnotesize UIQI &&\footnotesize PSNR  &\footnotesize SSIM &\footnotesize UIQI &&\footnotesize PSNR  &\footnotesize SSIM &\footnotesize UIQI &&\footnotesize PSNR  &\footnotesize SSIM &\footnotesize UIQI &&\footnotesize PSNR  &\footnotesize SSIM &\footnotesize UIQI& (s) \\\midrule
\multirow{10}{*}{\textit{foreman}}
&
Observed  &            3.96   &            0.010   &            0.006  & &            4.48   &            0.017   &            0.016  & &            5.05   &            0.025   &            0.027  & &            5.72   &            0.035   &            0.040  & &            6.51   &            0.047   &            0.056  &   0 \\
&HaLRTC    &            20.10   &            0.511   &            0.328  & &            23.88   &            0.700   &            0.562  & &            26.77   &            0.812   &            0.699  & &            29.35   &            0.883   &            0.790  & &            31.85   &            0.928   &            0.853  &   4 \\
&BCPF      &            23.58   &            0.610   &            0.458  & &            26.03   &            0.723   &            0.586  & &            27.27   &            0.769   &            0.638  & &            27.98   &            0.793   &            0.666  & &            28.42   &            0.808   &            0.682  &   829 \\
&TRLRF     & \color{greed} 24.63   &            0.616   & \color{greed} 0.536  & & \color{greed} 27.72   & \color{greed} 0.778   & \color{greed} 0.667  & &            29.01   &            0.830   &            0.719  & &            29.94   &            0.861   &            0.756  & &            30.95   &            0.887   &            0.788  &   387 \\
&TNN       &            23.71   &            0.606   &            0.489  & &            26.90   &            0.748   &            0.640  & &            29.13   &            0.824   &            0.720  & &            31.28   &            0.880   &            0.783  & &            33.47   &            0.921   &            0.835  &   30 \\
&DCTNN     &            24.27   & \color{greed} 0.632   &            0.519  & &            27.23   &            0.766   &            0.657  & & \color{greed} 29.54   & \color{greed} 0.843   & \color{greed} 0.740  & & \color{greed} 31.80   & \color{greed} 0.897   & \color{greed} 0.804  & & \color{greed} 34.07   & \color{greed} 0.934   & \color{greed} 0.854  &   21 \\
&FTNN      & \color{blue} 25.41   & \color{blue} 0.735   & \color{blue} 0.603  & & \color{blue} 28.51   & \color{blue} 0.849   & \color{blue} 0.741  & & \color{blue} 30.92   & \color{blue} 0.904   & \color{blue} 0.815  & & \color{blue} 33.17   & \color{blue} 0.937   & \color{blue} 0.862  & & \color{blue} 35.44   & \color{blue} 0.960   & \color{blue} 0.900  &   112 \\
&DTNN  & \color{red} 26.22   & \color{red} 0.799   & \color{red} 0.676  & & \color{red} 29.26   & \color{red} 0.875   & \color{red} 0.778  & & \color{red} 31.77   & \color{red} 0.917   & \color{red} 0.837  & & \color{red} 34.13   & \color{red} 0.946   & \color{red} 0.879  & & \color{red} 36.32   & \color{red} 0.964   & \color{red} 0.907  &   301 \\
\midrule
\multirow{10}{*}{\textit{carphone}}
&Observed  &            7.04   &            0.023   &            0.010  & &            7.56   &            0.039   &            0.026  & &            8.13   &            0.057   &            0.046  & &            8.81   &            0.077   &            0.070  & &            9.59   &            0.100   &            0.097  &   0 \\
 &HaLRTC    &            24.71   &            0.779   &            0.596  & &            28.57   &            0.883   &            0.751  & &            31.31   &            0.930   & \color{greed} 0.827  & &            33.67   & \color{greed} 0.956   & \color{greed} 0.875  & & \color{greed} 35.84   & \color{greed} 0.972   & \color{blue} 0.908  &   6 \\
 &BCPF      &            27.91   &            0.812   &            0.638  & &            29.82   &            0.864   &            0.703  & &            31.30   &            0.894   &            0.744  & &            32.25   &            0.911   &            0.767  & &            32.88   &            0.921   &            0.781  &   825 \\
 &TRLRF     & \color{blue} 29.44   & \color{greed} 0.841   & \color{greed} 0.691  & & \color{blue} 32.10   & \color{greed} 0.905   & \color{greed} 0.767  & & \color{blue} 33.58   & \color{greed} 0.930   &            0.805  & & \color{greed} 34.71   &            0.945   &            0.832  & &            35.63   &            0.955   &            0.853  &   414 \\
 &TNN       &            27.44   &            0.804   &            0.639  & &            30.00   &            0.873   &            0.725  & &            31.81   &            0.909   &            0.774  & &            33.44   &            0.934   &            0.813  & &            35.06   &            0.952   &            0.846  &   31 \\
 &DCTNN     &            28.21   &            0.829   &            0.669  & &            30.64   &            0.889   &            0.747  & &            32.37   &            0.920   &            0.792  & &            33.96   &            0.943   &            0.829  & &            35.57   &            0.959   &            0.859  &   21 \\
 &FTNN      & \color{greed} 29.16   & \color{blue} 0.880   & \color{blue} 0.740  & & \color{greed} 31.58   & \color{blue} 0.927   & \color{blue} 0.815  & & \color{greed} 33.43   & \color{blue} 0.949   & \color{blue} 0.855  & & \color{blue} 35.03   & \color{blue} 0.963   & \color{blue} 0.884  & & \color{blue} 36.61   & \color{blue} 0.973   & \color{greed} 0.906  &   118 \\
&DTNN  & \color{red} 29.46   & \color{red} 0.896   & \color{red} 0.763  & & \color{red} 32.89   & \color{red} 0.943   & \color{red} 0.838  & & \color{red} 35.49   & \color{red} 0.964   & \color{red} 0.879  & & \color{red} 37.57   & \color{red} 0.975   & \color{red} 0.904  & & \color{red} 39.35   & \color{red} 0.982   & \color{red} 0.922  &   275 \\
\midrule
\multirow{10}{*}{\textit{container}}&
 Observed  &            4.87   &            0.011   &            0.007  & &            5.38   &            0.021   &            0.018  & &            5.96   &            0.032   &            0.031  & &            6.63   &            0.045   &            0.047  & &            7.42   &            0.060   &            0.064  &   0 \\
 &HaLRTC    &            25.96   &            0.855   &            0.614  & &            30.58   &            0.935   &            0.779  & &            34.98   &            0.970   &            0.879  & &            39.56   &            0.986   &            0.937  & &            44.54   &            0.994   &            0.969  &   6 \\
 &BCPF      &            28.97   &            0.880   &            0.629  & &            33.38   &            0.926   &            0.707  & &            35.77   &            0.944   &            0.748  & &            37.81   &            0.954   &            0.777  & &            38.66   &            0.960   &            0.795  &   875 \\
 &TRLRF     & \color{red} 32.82   & \color{greed} 0.932   &            0.738  & & \color{greed} 37.55   &            0.961   &            0.817  & &            39.57   &            0.969   &            0.851  & &            40.64   &            0.974   &            0.875  & &            41.94   &            0.980   &            0.900  &   387 \\
 &TNN       &            30.04   &            0.909   &            0.715  & &            35.55   &            0.963   &            0.853  & &            38.81   &            0.979   &            0.905  & &            41.55   &            0.986   &            0.934  & &            44.01   &            0.991   &            0.954  &   29 \\
 &DCTNN     &            31.61   &            0.930   & \color{greed} 0.762  & & \color{blue} 38.27   & \color{greed} 0.977   & \color{greed} 0.892  & & \color{blue} 42.69   & \color{blue} 0.989   & \color{greed} 0.940  & & \color{blue} 45.91   & \color{blue} 0.993   & \color{blue} 0.960  & & \color{blue} 48.43   & \color{blue} 0.996   & \color{blue} 0.972  &   19 \\
 &FTNN      & \color{greed} 32.43   & \color{blue} 0.948   & \color{blue} 0.809  & &            37.38   & \color{blue} 0.978   & \color{blue} 0.904  & & \color{greed} 41.41   & \color{greed} 0.988   & \color{blue} 0.946  & & \color{greed} 44.42   & \color{greed} 0.992   & \color{greed} 0.958  & & \color{greed} 47.16   & \color{greed} 0.994   & \color{greed} 0.971  &   146 \\
 &DTNN  & \color{blue} 32.50   & \color{red} 0.956   & \color{red} 0.840  & & \color{red} 39.12   & \color{red} 0.985   & \color{red} 0.929  & & \color{red} 43.72   & \color{red} 0.992   & \color{red} 0.959  & & \color{red} 47.28   & \color{red} 0.995   & \color{red} 0.972  & & \color{red} 49.82   & \color{red} 0.997   & \color{red} 0.980  &   355 \\
\midrule
\multirow{10}{*}{\textit{highway}}&
 Observed  &            3.52   &            0.010   &            0.003  & &            4.04   &            0.015   &            0.008  & &            4.61   &            0.020   &            0.014  & &            5.28   &            0.027   &            0.021  & &            6.07   &            0.034   &            0.029  &   0 \\
 &HaLRTC    &            28.80   &            0.854   &            0.604  & &            31.55   & \color{greed} 0.909   & \color{greed} 0.730  & &            33.57   & \color{greed} 0.937   & \color{greed} 0.797  & &            35.26   & \color{greed} 0.954   & \color{blue} 0.844  & &            36.83   & \color{greed} 0.966   & \color{red} 0.880  &   5 \\
 &BCPF      &            29.96   &            0.840   &            0.593  & &            32.10   &            0.879   &            0.664  & &            33.17   &            0.897   &            0.693  & &            34.16   &            0.912   &            0.716  & &            34.57   &            0.918   &            0.729  &   714 \\
 &TRLRF     & \color{blue} 31.31   &            0.857   & \color{greed} 0.661  & & \color{blue} 33.81   &            0.905   &            0.729  & & \color{blue} 35.35   &            0.929   &            0.773  & & \color{blue} 36.40   &            0.943   &            0.806  & & \color{blue} 37.59   &            0.956   &            0.841  &   387 \\
 &TNN       &            30.19   &            0.852   &            0.631  & &            32.07   &            0.893   &            0.704  & &            33.57   &            0.917   &            0.753  & &            34.90   &            0.936   &            0.794  & &            36.26   &            0.951   &            0.832  &   26 \\
 &DCTNN     &            30.59   & \color{greed} 0.864   &            0.648  & &            32.35   &            0.899   &            0.715  & &            33.79   &            0.922   &            0.762  & &            35.12   &            0.939   &            0.802  & &            36.49   &            0.954   &            0.838  &   19 \\
 &FTNN      & \color{greed} 31.23   & \color{blue} 0.893   & \color{blue} 0.693  & & \color{greed} 33.09   & \color{blue} 0.926   & \color{blue} 0.764  & & \color{greed} 34.63   & \color{blue} 0.944   & \color{blue} 0.808  & & \color{greed} 35.93   & \color{blue} 0.956   & \color{greed} 0.842  & & \color{greed} 37.25   & \color{blue} 0.966   & \color{greed} 0.875  &   123 \\
 &DTNN  & \color{red} 31.70   & \color{red} 0.902   & \color{red} 0.711  & & \color{red} 33.96   & \color{red} 0.931   & \color{red} 0.774  & & \color{red} 35.81   & \color{red} 0.948   & \color{red} 0.815  & & \color{red} 37.34   & \color{red} 0.959   & \color{red} 0.849  & & \color{red} 38.71   & \color{red} 0.968   & \color{blue} 0.877  &   223 \\
\bottomrule
\end{tabular}
\label{Tab-Video-Foreman}
\end{table*}

\subsection{Video Data}

In this subsection, we test our method for the video data completion and select four videos\footnote{\scriptsize Videos available at  http://trace.eas.asu.edu/yuv/.} named ``\textit{foreman}'' ``\textit{carphone}'' ``\textit{highway}'' and ``\textit{container}'' of the size $144\times176\times50$ (height$\times$width$\times$frame) to conduct the comparisons. The sampling rate (SR) varies from 10\% to 50\%. We compute the peak signal-to-noise ratio (PSNR), the structural similarity index (SSIM) \cite{ssim}, and the universal image quality index (UIQI) \cite{wang2002universal} of the results by different methods. Higher values of these three quality metrics indicate better completion performances.

\begin{figure*}[!t]
\setlength{\tabcolsep}{2pt}
\renewcommand\arraystretch{1}
\centering\footnotesize
\begin{tabular}{ccccc ccccc}
Observed & HaLRTC \cite{Liu2013PAMItensor} & BCPF \cite{zhao2015bayesian} & TRLRF \cite{yuan2019tensor}&
 TNN \cite{zhang2017exact} & DCTNN \cite{lu2019low} & FTNN \cite{jiang2019framelet} & DTNN &  Original\\
\includegraphics[width=0.1\linewidth]{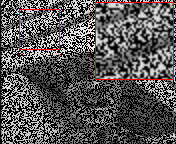} &
\includegraphics[width=0.1\linewidth]{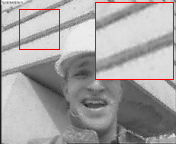} &
\includegraphics[width=0.1\linewidth]{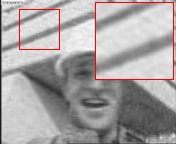} &
\includegraphics[width=0.1\linewidth]{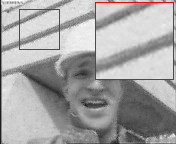} &
\includegraphics[width=0.1\linewidth]{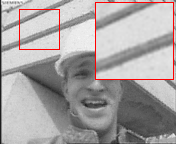} &
\includegraphics[width=0.1\linewidth]{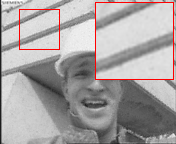} &
\includegraphics[width=0.1\linewidth]{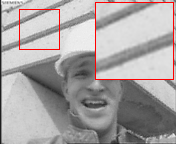} &
\includegraphics[width=0.1\linewidth]{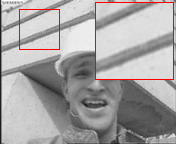} &
\includegraphics[width=0.1\linewidth]{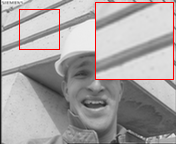} \\

\includegraphics[width=0.1\linewidth]{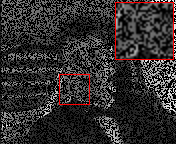} &
\includegraphics[width=0.1\linewidth]{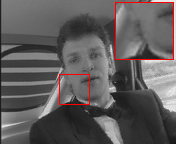} &
\includegraphics[width=0.1\linewidth]{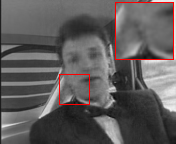} &
\includegraphics[width=0.1\linewidth]{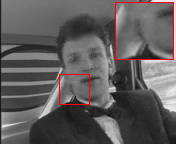} &
\includegraphics[width=0.1\linewidth]{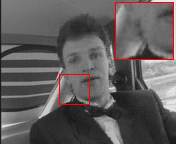} &
\includegraphics[width=0.1\linewidth]{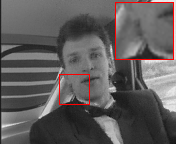} &
\includegraphics[width=0.1\linewidth]{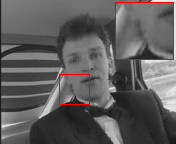} &
\includegraphics[width=0.1\linewidth]{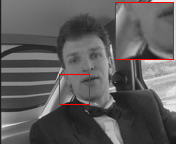} &
\includegraphics[width=0.1\linewidth]{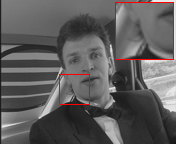}\\

\includegraphics[width=0.1\linewidth]{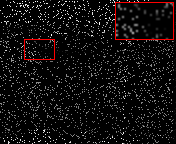} &
\includegraphics[width=0.1\linewidth]{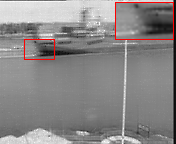} &
\includegraphics[width=0.1\linewidth]{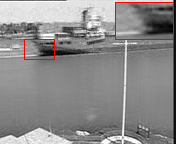} &
\includegraphics[width=0.1\linewidth]{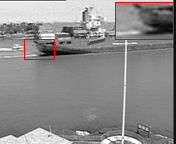} &
\includegraphics[width=0.1\linewidth]{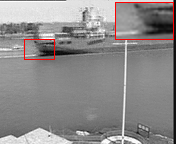} &
\includegraphics[width=0.1\linewidth]{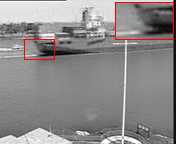} &
\includegraphics[width=0.1\linewidth]{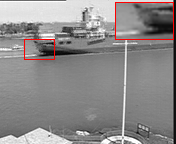} &
\includegraphics[width=0.1\linewidth]{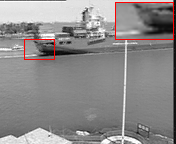} &
\includegraphics[width=0.1\linewidth]{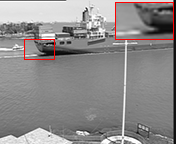} \\

\includegraphics[width=0.1\linewidth]{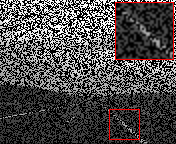} &
\includegraphics[width=0.1\linewidth]{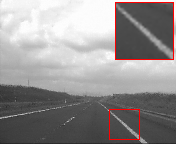} &
\includegraphics[width=0.1\linewidth]{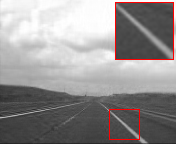} &
\includegraphics[width=0.1\linewidth]{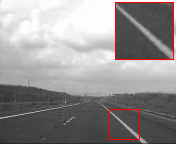} &
\includegraphics[width=0.1\linewidth]{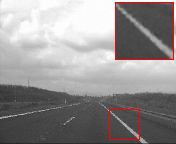} &
\includegraphics[width=0.1\linewidth]{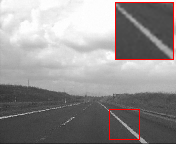} &
\includegraphics[width=0.1\linewidth]{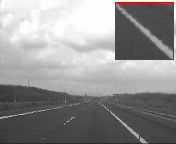} &
\includegraphics[width=0.1\linewidth]{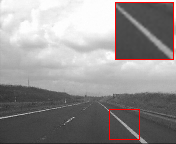} &
\includegraphics[width=0.1\linewidth]{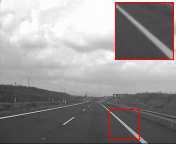} \\

\end{tabular}\vspace{-1mm}
\caption{One frame of the results on the \textbf{video} data. From top to bottom: The 22-th frame of ``\textit{foreman}'' (SR = 50\%), the 5-th frame of ``carphone'' (SR = 50\%), the 39-th frame of ``container'' (SR = 10\%), and the 48-th frame of ``highway'' (SR = 50\%).}
\label{videoframe}
\end{figure*}

In Tab. \ref{Tab-Video-Foreman}, we report the quantitative metrics of the results obtained by different methods and the average running time on the video data. From Tab. \ref{Tab-Video-Foreman}, it can be found that the results by TRLRF are promising when the sampling rate is low. The performance of FTNN is better than TNN and DCTNN for the video ``\textit{foreman}'', while DCTNN exceeds FTNN and TNN for the video ``\textit{container}''. This reveals the predefined transformations lack flexibility. Meanwhile, with minor exceptions, our DTNN achieves the best performance for different sampling rates, illustrating the {\red superiority} of the data adaptive dictionary.

\begin{table*}[!t]
\renewcommand\arraystretch{0.75}\setlength{\tabcolsep}{1.5pt}
\centering
\caption{PSNR, SSIM, and SAM of results by different methods with different sampling rates on the \textbf{HSI} data. The {\color{red}best}, the {\color{blue}second best}, and the {\color{greed} third best} values are respectively highlighted by {\color{red}red}, {\color{blue}blue}, and {\color{greed}green} colors.}\vspace{-2mm}
\begin{tabular}{llccccccccccccccccccccccccc}\toprule
\multirow{3}{*}{HSI}&SR    & \multicolumn{3}{c}{5\%} && \multicolumn{3}{c}{10\%} && \multicolumn{3}{c}{20\%} && \multicolumn{3}{c}{30\%} && \multicolumn{3}{c}{40\%} && \multicolumn{3}{c}{50\%}  &&\multirow{2}{*}{Time} \\\cmidrule(r){2-25}

&Method & \footnotesize PSNR  &\footnotesize SSIM &\footnotesize SAD &&\footnotesize PSNR  &\footnotesize SSIM &\footnotesize SAD &&\footnotesize PSNR  &\footnotesize SSIM &\footnotesize SAD &&\footnotesize PSNR  &\footnotesize SSIM &\footnotesize SAD &&\footnotesize PSNR  &\footnotesize SSIM &\footnotesize SAD&&\footnotesize PSNR  &\footnotesize SSIM &\footnotesize SAD&&(s) \\\midrule
\multirow{8}{*}{Pavia}&
 Observed  &            12.19   &            0.020   &            1.355  & &            12.43   &            0.036   &            1.254  & &            12.94   &            0.070   &            1.110  & &            13.52   &            0.108   &            0.993  & &            14.18   &            0.149   &            0.887  &&            14.98   &            0.196   &            0.785  & &   0 \\
\multirow{8}{*}{City}& HaLRTC    &            22.95   &            0.596   &            0.126  & &            27.67   &            0.835   &            0.095  & &            35.68   &            0.970   &            0.048  & &            43.11   &            0.994   &            0.024  & &            51.84   & \color{greed} 0.999   & \color{greed} 0.010  &&            56.13   & \color{greed} 1.000   & \color{greed} 0.006  & &   24 \\
& BCPF      &            29.12   &            0.850   &            0.100  & &            33.00   &            0.927   &            0.077  & &            37.71   &            0.970   &            0.051  & &            39.51   &            0.980   &            0.043  & &            40.12   &            0.982   &            0.041  &&            40.78   &            0.985   &            0.038  & &   3149 \\
& TRLRF     & \color{blue} 33.75   & \color{blue} 0.936   & \color{blue} 0.068  & &            37.03   &            0.967   & \color{greed} 0.051  & &            38.97   &            0.978   &            0.044  & &            40.13   &            0.983   &            0.040  & &            41.10   &            0.986   &            0.036  &&            42.01   &            0.989   &            0.033  & &   1153 \\
& TNN       &            26.08   &            0.739   &            0.147  & &            32.49   &            0.918   &            0.099  & &            38.18   &            0.968   &            0.064  & &            41.89   &            0.982   &            0.048  & &            45.01   &            0.989   &            0.037  &&            48.13   &            0.993   &            0.029  & &   97 \\
& DCTNN     &            29.72   &            0.870   &            0.098  & & \color{greed} 38.26   & \color{blue} 0.978   & \color{blue} 0.042  & & \color{blue} 48.54   & \color{blue} 0.998   & \color{blue} 0.015  & & \color{blue} 54.79   & \color{blue} 0.999   & \color{blue} 0.008  & & \color{blue} 59.20   & \color{blue} 1.000   & \color{blue} 0.005  && \color{blue} 62.97   & \color{blue} 1.000   & \color{blue} 0.004  & &   66 \\
& FTNN      & \color{greed} 33.51   & \color{greed} 0.936   & \color{greed} 0.076  & & \color{blue} 38.60   & \color{greed} 0.974   &            0.053  & & \color{greed} 45.37   & \color{greed} 0.991   & \color{greed} 0.033  & & \color{greed} 49.73   & \color{greed} 0.995   & \color{greed} 0.024  & & \color{greed} 54.99   &            0.997   &            0.017  && \color{greed} 57.73   &            0.998   &            0.013  & &   431 \\
& DTNN  & \color{red} 34.26   & \color{red} 0.953   & \color{red} 0.052  & & \color{red} 40.86   & \color{red} 0.989   & \color{red} 0.026  & & \color{red} 53.20   & \color{red} 0.999   & \color{red} 0.008  & & \color{red} 65.00   & \color{red} 1.000   & \color{red} 0.002  & & \color{red} 67.15   & \color{red} 1.000   & \color{red} 0.002  && \color{red} 77.16   & \color{red} 1.000   & \color{red} 0.001  & &   962 \\
\midrule
\multirow{8}{*}{Washing-}&
Observed  &            12.45   &            0.028   &            1.353  & &            12.68   &            0.053   &            1.254  & &            13.19   &            0.108   &            1.110  & &            13.77   &            0.169   &            0.993  & &            14.44   &            0.234   &            0.887  &&            15.23   &            0.304   &            0.785  & &   0 \\
\multirow{8}{*}{ton DC}& HaLRTC    &            23.24   &            0.713   &            0.208  & &            29.36   &            0.906   &            0.125  & & \color{greed} 38.21   & \color{greed} 0.983   & \color{greed} 0.062  & & \color{greed} 44.76   & \color{blue} 0.996   & \color{blue} 0.034  & & \color{blue} 49.67   & \color{blue} 0.999   & \color{blue} 0.020 & & \color{blue} 53.12   & \color{blue} 0.999   & \color{blue} 0.015  & &   49 \\
& BCPF      &            28.65   &            0.877   &            0.155  & &            32.06   &            0.939   &            0.118  & &            34.75   &            0.964   &            0.093  & &            35.70   &            0.971   &            0.086  & &            35.98   &            0.972   &            0.083  &&            36.03   &            0.973   &            0.083  & &   5566 \\
& TRLRF     & \color{greed} 31.74   & \color{greed} 0.934   & \color{greed} 0.121  & & \color{greed} 33.64   & \color{greed} 0.955   & \color{greed} 0.102  & &            34.98   &            0.966   &            0.091  & &            35.86   &            0.972   &            0.084  & &            36.73   &            0.976   &            0.078  &&            37.69   &            0.981   &            0.071  & &   1931 \\
& TNN       &            22.34   &            0.657   &            0.260  & &            30.19   &            0.915   &            0.142  & &            36.56   &            0.974   &            0.085  & &            40.10   &            0.987   &            0.062  & &            43.03   &            0.992   &            0.047  &&            45.68   &            0.995   &            0.036  & &   159 \\
& DCTNN     &            27.09   &            0.840   &            0.182  & &            32.59   &            0.946   &            0.116  & &            38.16   &            0.982   &            0.072  & &            41.85   &            0.992   &            0.051  & &            44.86   &            0.995   &            0.038  &&            47.50   &            0.997   &            0.029  & &   110 \\
& FTNN      & \color{blue} 32.17   & \color{blue} 0.946   & \color{blue} 0.107  & & \color{blue} 37.03   & \color{blue} 0.979   & \color{blue} 0.075  & & \color{blue} 42.96   & \color{blue} 0.992   & \color{blue} 0.048  & & \color{blue} 46.64   & \color{greed} 0.996   & \color{greed} 0.036  & & \color{greed} 49.46   & \color{greed} 0.997   & \color{greed} 0.028&  & \color{greed} 52.06   & \color{greed} 0.998   & \color{greed} 0.021  & &   805 \\
&     DTNN  & \color{red} 34.21   & \color{red} 0.969   & \color{red} 0.075  & & \color{red} 39.85   & \color{red} 0.992   & \color{red} 0.044  & & \color{red} 45.01   & \color{red} 0.997   & \color{red} 0.033  & & \color{red} 47.56   & \color{red} 0.998   & \color{red} 0.029  & & \color{red} 53.35   & \color{red} 0.999   & \color{red} 0.015  && \color{red} 58.59   & \color{red} 1.000   & \color{red} 0.009  & &   1837 \\
\bottomrule
    \end{tabular}
\label{table-HSI-Pavia}
\end{table*}

Fig. \ref{videoframe} exhibits one frame of the results by different methods on the video data. From the enlarged area, it can be found that our DTNN well restores edges in ``\textit{foreman}'' and ``\textit{highway}'', the hair in ``\textit{carphone}'', and the ship's outline in `\textit{container}''. The homogeneous areas are also protected by our method. We can conclude that the visual effect of our method is the best.
\subsection{Hyperspectral Images}

In this subsection, 2 HSIs, i.e.
a subimage of Pavia City Center dataset\footnote{\scriptsize http://www.ehu.eus/ccwintco/index.php?title=Hyperspectral\_Remote\_Sensing\_Scenes} of the size $200\times200\times80$ (height$\times$width$\times$band),
and a subimage of Washington DC Mall dataset\footnote{\scriptsize https://engineering.purdue.edu/\~biehl/MultiSpec/hyperspectral.html} of the size $256\times256\times{\red 160}$
are adopted as the testing data.
Since the redundancy between HSIs' slices is so high that all the methods perform very well with SR=50\%, we add the case with SR=5\%. Thus, the sampling rates vary from 5\% to 50\%.
Three numerical metrics, consisting of PSNR, SSIM, and the mean Spectral Angle Mapper (SAM) \cite{yuhas1993determination} are selected to quantitatively measure the reconstructed results. Lower values of SAM indicate better reconstructions.

In Tab. \ref{table-HSI-Pavia}, we show the quantitative comparisons of different methods on HSIs. FTNN and TRLRF perform well for the low sampling rate. We can also see that DCTNN and FTNN alternatively achieves the second best place in many cases, showing that DCT and framelet transformation fit the HSI  data better than DFT.
For different metrics, our DTNN obtains the best values in all cases.
As sampling rates arise, the {\red superiority} of our method over compared methods is more evident. For example, when dealing with Pavia City Center, the margins are at least 7.95 dB and 14.19 dB  for PSNR when SR is 40\% and 50\% , respectively.
We attribute this to the fact that,the dictionary could be learned with better ability to express the data  when the sampling rate is high.

We display the  pseudo-color images (using three bands to compose the RGB channels) of the reconstructed HSIs in Fig. \ref{hsiband}. The similarity of the color reflects the fidelity along the spectral direction, which is of vital importance in applications of HSIs. It can be found that the color distortion occurs in the results by TNN. From the enlarged orange and red boxes, we can see that DTNN outperforms compared methods considering the spatial structures and details. 

\begin{figure*}[!t]
\setlength{\tabcolsep}{2pt}
\renewcommand\arraystretch{0.7}
\centering\footnotesize
\begin{tabular}{ccccc ccccc}
Observed & HaLRTC \cite{Liu2013PAMItensor} & BCPF \cite{zhao2015bayesian} & TRLRF \cite{yuan2019tensor}&
 TNN \cite{zhang2017exact} & DCTNN \cite{lu2019low} & FTNN \cite{jiang2019framelet} & DTNN &  Original\\
\includegraphics[width=0.1\linewidth]{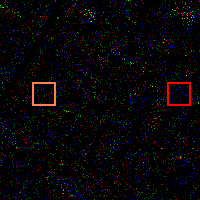} &
\includegraphics[width=0.1\linewidth]{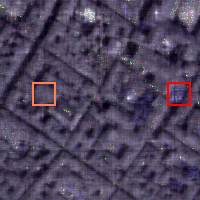} &
\includegraphics[width=0.1\linewidth]{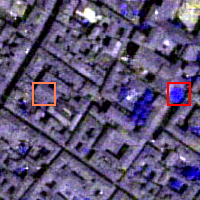} &
\includegraphics[width=0.1\linewidth]{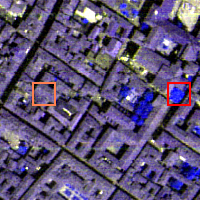} &
\includegraphics[width=0.1\linewidth]{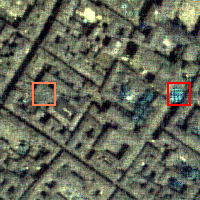} &
\includegraphics[width=0.1\linewidth]{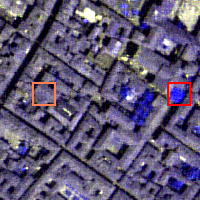} &
\includegraphics[width=0.1\linewidth]{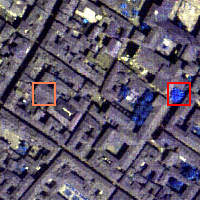} &
\includegraphics[width=0.1\linewidth]{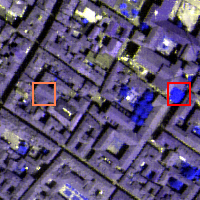} &
\includegraphics[width=0.1\linewidth]{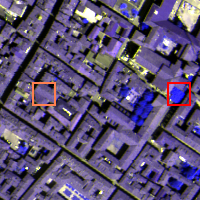} \\

\includegraphics[width=0.08\linewidth]{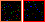} &
\includegraphics[width=0.08\linewidth]{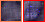} &
\includegraphics[width=0.08\linewidth]{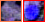} &
\includegraphics[width=0.08\linewidth]{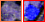} &
\includegraphics[width=0.08\linewidth]{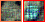} &
\includegraphics[width=0.08\linewidth]{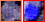} &
\includegraphics[width=0.08\linewidth]{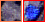} &
\includegraphics[width=0.08\linewidth]{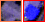} &
\includegraphics[width=0.08\linewidth]{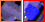} \\\\

\includegraphics[width=0.1\linewidth]{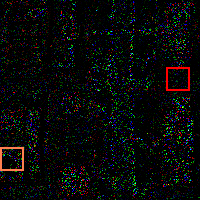} &
\includegraphics[width=0.1\linewidth]{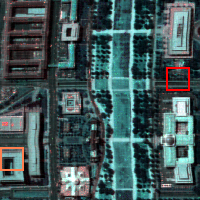} &
\includegraphics[width=0.1\linewidth]{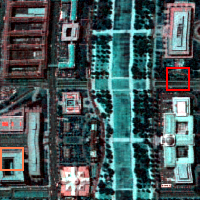} &
\includegraphics[width=0.1\linewidth]{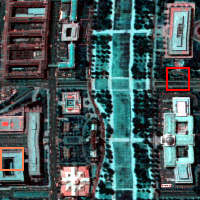} &
\includegraphics[width=0.1\linewidth]{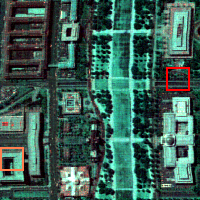} &
\includegraphics[width=0.1\linewidth]{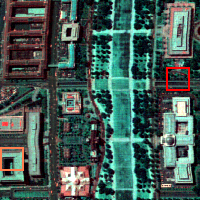} &
\includegraphics[width=0.1\linewidth]{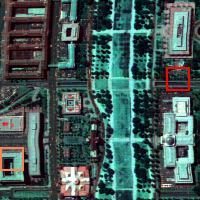} &
\includegraphics[width=0.1\linewidth]{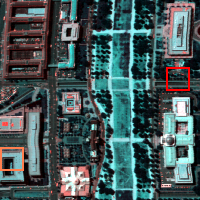} &
\includegraphics[width=0.1\linewidth]{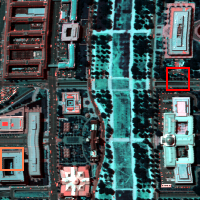} \\

\includegraphics[width=0.08\linewidth]{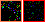} &
\includegraphics[width=0.08\linewidth]{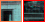} &
\includegraphics[width=0.08\linewidth]{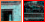} &
\includegraphics[width=0.08\linewidth]{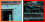} &
\includegraphics[width=0.08\linewidth]{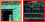} &
\includegraphics[width=0.08\linewidth]{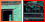} &
\includegraphics[width=0.08\linewidth]{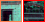} &
\includegraphics[width=0.08\linewidth]{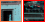} &
\includegraphics[width=0.08\linewidth]{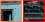} \\\\
\end{tabular}\vspace{-2mm}
\caption{The pseudo-color images and the corresponding enlarged areas of the results by different methods. Top: Pavia City Center (R-4 G-12 B-68) with SR = 5\%. Bottom: Washington DC Mall (R-1   G-113   B-116) with SR = 10\%.}
\label{hsiband}
\end{figure*}

\begin{table*}[!t]
\renewcommand\arraystretch{0.7}\setlength{\tabcolsep}{3pt}
\centering
\caption{RMSE and MAPE of results by different methods with different sampling rates on the \textbf{traffic} data. The {\color{red}best}, the {\color{blue}second best}, and the {\color{greed} third best} values are respectively highlighted by {\color{red}red}, {\color{blue}blue}, and {\color{greed}green} colors.}\vspace{-2mm}
\begin{tabular}{lccccccccccccccccccc}\toprule
SR    & \multicolumn{2}{c}{5\%} && \multicolumn{2}{c}{10\%} && \multicolumn{2}{c}{15\%} && \multicolumn{2}{c}{20\%} && \multicolumn{2}{c}{25\%} && \multicolumn{2}{c}{30\%}&\multirow{2}{*}{Time} \\\cmidrule(r){1-18}

Method & RMSE  & MAPE  && RMSE  & MAPE  && RMSE  & MAPE  && RMSE  & MAPE  && RMSE  & MAPE  && RMSE  & MAPE & (s)\\\midrule
 Observed  &               0.9148   &               95.71 \% & &               0.8942   &               91.44 \% & &               0.8731   &               87.15 \% & &               0.8514   &               82.87 \% & &               0.8291   &               78.56 \% & &               0.8061   &               74.26 \% &   0 \\
 HaLRTC    &               0.3592   &               17.51 \% & &               0.3581   &               16.72 \% & &               0.3575   &               16.26 \% & &               0.3571   &               15.94 \% & &               0.3569   &               15.68 \% & &               0.3566   &               15.47 \% &   13 \\
 BCPF      &               0.3594   &               17.25 \% & &               0.3576   &               16.45 \% & &               0.3571   &               16.13 \% & &               0.3568   &               15.97 \% & &               0.3567   &               15.84 \% & &               0.3566   &               15.76 \% &   218 \\
 TRLRF     &               0.1781   &                9.09 \% & &               0.1753   &                8.62 \% & &               0.1749   &                8.42 \% & &               0.1747   &                8.29 \% & &               0.1745   &                8.20 \% & &               0.1744   &                8.08 \% &   107 \\
 TNN       &               0.0509   &                3.63 \% & &               0.0387   &                2.75 \% & & \color{greed} 0.0336   &                2.33 \% & & \color{greed} 0.0304   & \color{greed}  2.03 \% & & \color{greed} 0.0283   & \color{greed}  1.82 \% & & \color{greed} 0.0267   & \color{greed}  1.65 \% &   27 \\
 DCTNN     & \color{greed} 0.0480   & \color{greed}  3.38 \% & & \color{greed} 0.0387   & \color{greed}  2.71 \% & &               0.0338   & \color{greed}  2.30 \% & &               0.0315   &                2.07 \% & &               0.0296   &                1.88 \% & &               0.0278   &                1.70 \% &   18 \\
 FTNN      & \color{blue}  0.0452   & \color{blue}   3.33 \% & & \color{blue}  0.0358   & \color{blue}   2.31 \% & & \color{blue}  0.0316   & \color{blue}   1.96 \% & & \color{blue}  0.0287   & \color{blue}   1.69 \% & & \color{blue}  0.0265   & \color{blue}   1.49 \% & & \color{blue}  0.0247   & \color{blue}   1.32 \% &   129 \\
     DTNN  & \color{red}   0.0428   & \color{red}    2.76 \% & & \color{red}   0.0354   & \color{red}    2.19 \% & & \color{red}   0.0305   & \color{red}    1.84 \% & & \color{red}   0.0278   & \color{red}    1.60 \% & & \color{red}   0.0256   & \color{red}    1.42 \% & & \color{red}   0.0237   & \color{red}    1.26 \% &   264 \\
\bottomrule
\end{tabular}
\label{table-traffic}
\end{table*}

\begin{figure*}[!t]
\setlength{\tabcolsep}{3pt}
\renewcommand\arraystretch{1}
\centering
\includegraphics[width=0.75\linewidth]{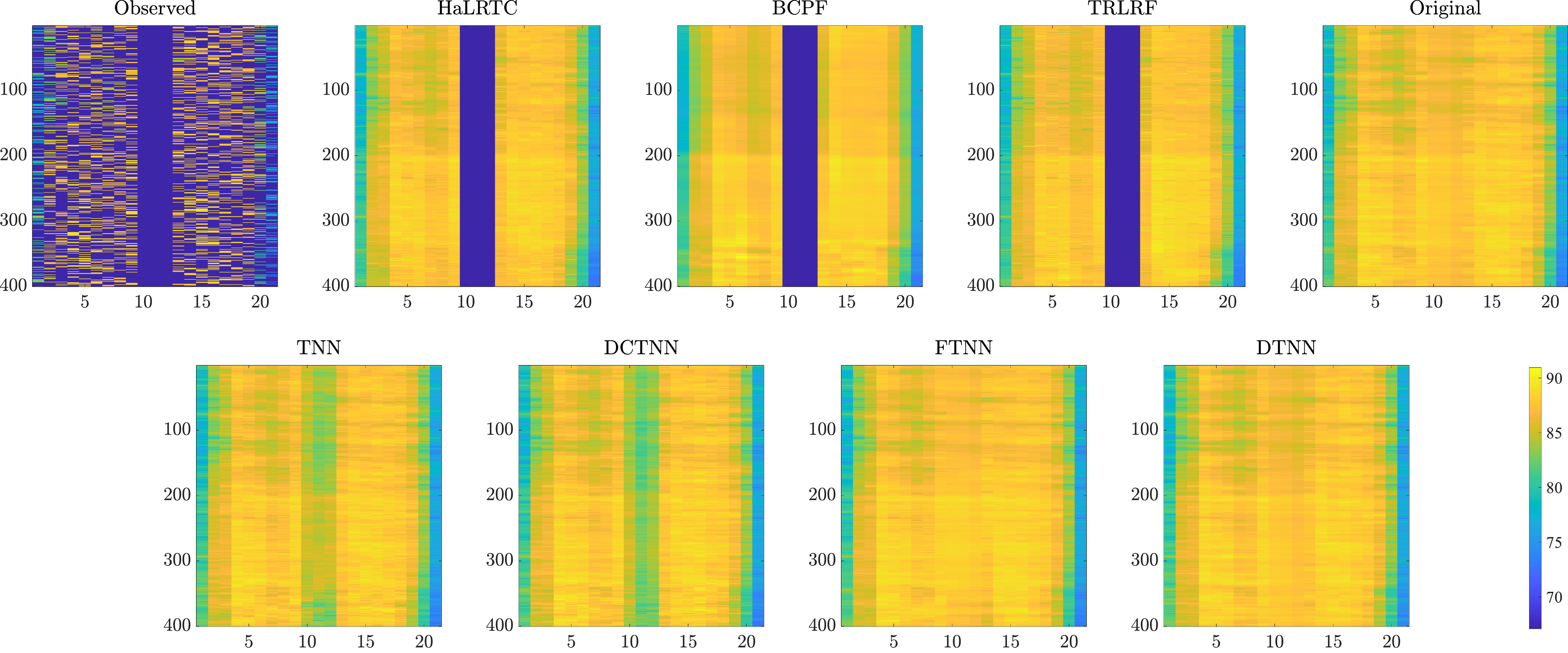}\vspace{-3mm}
\caption{The 88-th lateral slice of the reconstructions by different methods on the \textbf{traffic} data (SR = 30\%).}
\label{fig-traffic}
\end{figure*}

\subsection{Traffic Data}
In this subsection, we test all the methods on the traffic data\footnote{\scriptsize https://gtl.inrialpes.fr/data\_download}, which is provided by Grenoble Traffic Lab (GTL).
A set of traffic speed data of 207 days (April 1, 2015 to October 24, 2015), 1440 time windows\footnote{\scriptsize The sampling period is 1 minute, so there are $60\times24 = 1440$ time windows for each day.}, and 21 detection points, is downloaded and constitutes a third-order tensor of the size ${1440\times 207 \times 21}$. {\red To reduce the time consumption, a subset of the data with the size $400\times 200 \times 21$ corresponding to the first continuous 400 time windows in a day and the first 200 days is manually clipped as the ground truth complete testing data.}
We select the root mean square error (RMSE\footnote{\scriptsize
$
\text{RMSE} = \sqrt{\left(\sum_{ijk}(\mathcal{X}^\text{Rec}_{ijk}-\mathcal{X}_{ijk}^\text{GT})^2\right)/n_1 n_2 n_3}
$})
and the mean absolute percentage error (MAPE\footnote{\scriptsize
$
\text{MAPE} = \frac{1}{n_1n_2n_3}\sum_{ijk}(|\mathcal{X}^\text{Rec}_{ijk}-\mathcal{X}_{ijk}^\text{GT}|/\mathcal{X}_{ijk}^\text{GT})\times 100\%
$. This index is a measure of prediction accuracy, usually expressing accuracy as a percentage.
}) to quantitatively measure the quality of the results. {\red Lower} values of RMSE and MAPE indicate better reconstructions.
After random sampling the elements with $\text{SR}\in\{5\%,10\%,15\%,\cdots,30\%\}$, 3 adjacent frontal slices in a random location are set as unobserved.
This is to simulate the situations in which some detectors are broken.
The 200-th lateral slice of the observation is shown in the top-left of Fig. \ref{fig-traffic}, the missing slices corresponding to the blue columns.

\begin{table*}[!t]\renewcommand\arraystretch{0.7}\setlength{\tabcolsep}{4pt}
\centering
\caption{PSNR, SSIM, and UIQI of results by different methods with different sampling rates on the \textbf{MRI} data. The {\color{red}best}, the {\color{blue}second best}, and the {\color{greed} third best} values are respectively highlighted by {\color{red}red}, {\color{blue}blue}, and {\color{greed}green} colors.}\vspace{-2mm}
\begin{tabular}{lccccccccccccccccccccccc}\toprule
SR    & \multicolumn{3}{c}{10\%} && \multicolumn{3}{c}{20\%} && \multicolumn{3}{c}{30\%} && \multicolumn{3}{c}{40\%} && \multicolumn{3}{c}{50\%} &\multirow{2}{*}{Time} \\\cmidrule(r){1-20}

Method & \footnotesize PSNR  &\footnotesize SSIM &\footnotesize UIQI &&\footnotesize PSNR  &\footnotesize SSIM &\footnotesize UIQI &&\footnotesize PSNR  &\footnotesize SSIM &\footnotesize UIQI &&\footnotesize PSNR  &\footnotesize SSIM &\footnotesize UIQI &&\footnotesize PSNR  &\footnotesize SSIM &\footnotesize UIQI& (s) \\\midrule
 Observed  &            8.09   &            0.043   &            0.020  & &            8.60   &            0.070   &            0.050  & &            9.18   &            0.099   &            0.086  & &            9.85   &            0.132   &            0.127  & &            10.64   &            0.167   &            0.173  &   0 \\
 HaLRTC    &            18.34   &            0.436   &            0.349  & &            22.48   &            0.651   &            0.606  & &            26.01   &            0.794   &            0.756  & &            29.19   &            0.880   &            0.840  & &            32.13   &            0.931   &            0.889  &   11 \\
 TRLRF     & \color{greed} 23.89   &            0.637   &            0.606  & &            25.48   &            0.720   &            0.682  & &            26.36   &            0.760   &            0.722  & &            27.17   &            0.794   &            0.756  & &            28.06   &            0.825   &            0.786  &   921 \\
 BCPF      &            22.63   &            0.574   &            0.518  & &            24.70   &            0.678   &            0.631  & &            25.37   &            0.710   &            0.663  & &            25.55   &            0.720   &            0.671  & &            25.71   &            0.727   &            0.678  &   2430 \\
 TNN       &            22.41   &            0.577   &            0.550  & &            27.12   &            0.789   &            0.757  & &            30.01   &            0.874   &            0.833  & &            32.55   &            0.922   &            0.876  & &            35.01   &            0.953   &            0.905  &   60 \\
 DCTNN     &            23.79   & \color{greed} 0.644   & \color{greed} 0.617  & & \color{greed} 27.63   & \color{greed} 0.808   & \color{greed} 0.773  & & \color{greed} 30.56   & \color{greed} 0.888   & \color{greed} 0.844  & & \color{greed} 33.15   & \color{greed} 0.932   & \color{greed} 0.885  & & \color{greed} 35.62   & \color{greed} 0.960   & \color{greed} 0.911  &   45 \\
FTNN      & \color{blue} 25.15   & \color{blue} 0.743   & \color{blue} 0.695  & & \color{blue} 29.02   & \color{blue} 0.872   & \color{blue} 0.825  & & \color{blue} 31.96   & \color{blue} 0.928   & \color{blue} 0.883  & & \color{blue} 34.49   & \color{blue} 0.958   & \color{blue} 0.916  & & \color{blue} 36.89   & \color{blue} 0.975   & \color{red} 0.937  &   253 \\

 DTNN  & \color{red} 27.19   & \color{red} 0.835   & \color{red} 0.790  & & \color{red} 31.03   & \color{red} 0.917   & \color{red} 0.870  & & \color{red} 33.65   & \color{red} 0.949   & \color{red} 0.902  & & \color{red} 35.83   & \color{red} 0.966   & \color{red} 0.920  & & \color{red} 37.91   & \color{red} 0.978   & \color{blue} 0.934  &   568 \\

\bottomrule
    \end{tabular}%
\label{table-MRI}
\end{table*}
\begin{figure*}[!t]
\setlength{\tabcolsep}{2pt}
\renewcommand\arraystretch{1}
\centering
\begin{tabular}{ccccc ccccc}
\multicolumn{2}{c}{Observed} & \multicolumn{2}{c}{HaLRTC \cite{Liu2013PAMItensor}} & \multicolumn{2}{c}{BCPF \cite{zhao2015bayesian}} & \multicolumn{2}{c}{TRLRF \cite{yuan2019tensor}}& \multicolumn{2}{c}{Original}\\
\multicolumn{2}{c}{\includegraphics[width=0.18\linewidth]{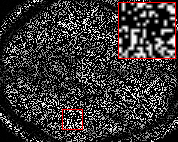}} &
\multicolumn{2}{c}{\includegraphics[width=0.18\linewidth]{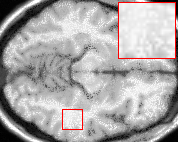}} &
\multicolumn{2}{c}{\includegraphics[width=0.18\linewidth]{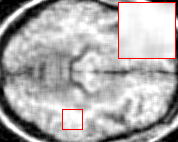}} &
\multicolumn{2}{c}{\includegraphics[width=0.18\linewidth]{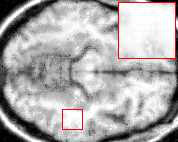}} &
\multicolumn{2}{c}{\includegraphics[width=0.18\linewidth]{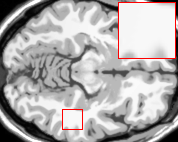}} \vspace{1mm}\\

\hspace{0.09\linewidth} &  \multicolumn{2}{c}{TNN \cite{zhang2017exact}} & \multicolumn{2}{c}{DCTNN \cite{lu2019low}} & \multicolumn{2}{c}{FTNN \cite{jiang2019framelet}} & \multicolumn{2}{c}{DTNN} \\
&
\multicolumn{2}{c}{\includegraphics[width=0.18\linewidth]{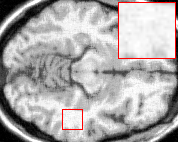}} &
\multicolumn{2}{c}{\includegraphics[width=0.18\linewidth]{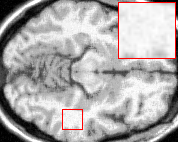}} &
\multicolumn{2}{c}{\includegraphics[width=0.18\linewidth]{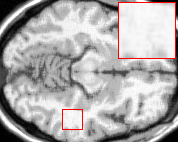}} &
\multicolumn{2}{c}{\includegraphics[width=0.18\linewidth]{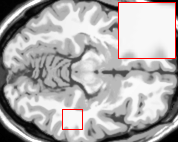}} \\
\end{tabular}\vspace{-1mm}
\caption{The 61-th frontal slice of the results on the \textbf{MRI} data  by different methods (SR = 30\%).}
\label{fig-MRI}
\end{figure*}

Tab. \ref{table-traffic} gives the quantitative metrics of the results by different methods with different sampling rates. We can find that the capabilities of HaLRTC, BCPF, and TRLRF is limited and this phenomenon accord with the visual results shown in Fig. \ref{fig-traffic}. The effectiveness of these three methods is severely affected due to the missing frontal slices. TNN and DCTNN get better metrics while their performance is also not well considering the location of missing slices. FTNN and DTNN recover the rough structure of the missing slices and the metrics of their results also achieve the best and the second best places. The reconstruction of our DTNN in the area of missing slices is closer to the original data than FTNN.
\subsection{MRI Data}
\begin{figure*}[!t]
\setlength{\tabcolsep}{6pt}
\renewcommand\arraystretch{0.8}
\centering
\begin{tabular}{cccc}
\includegraphics[width=0.22\linewidth]{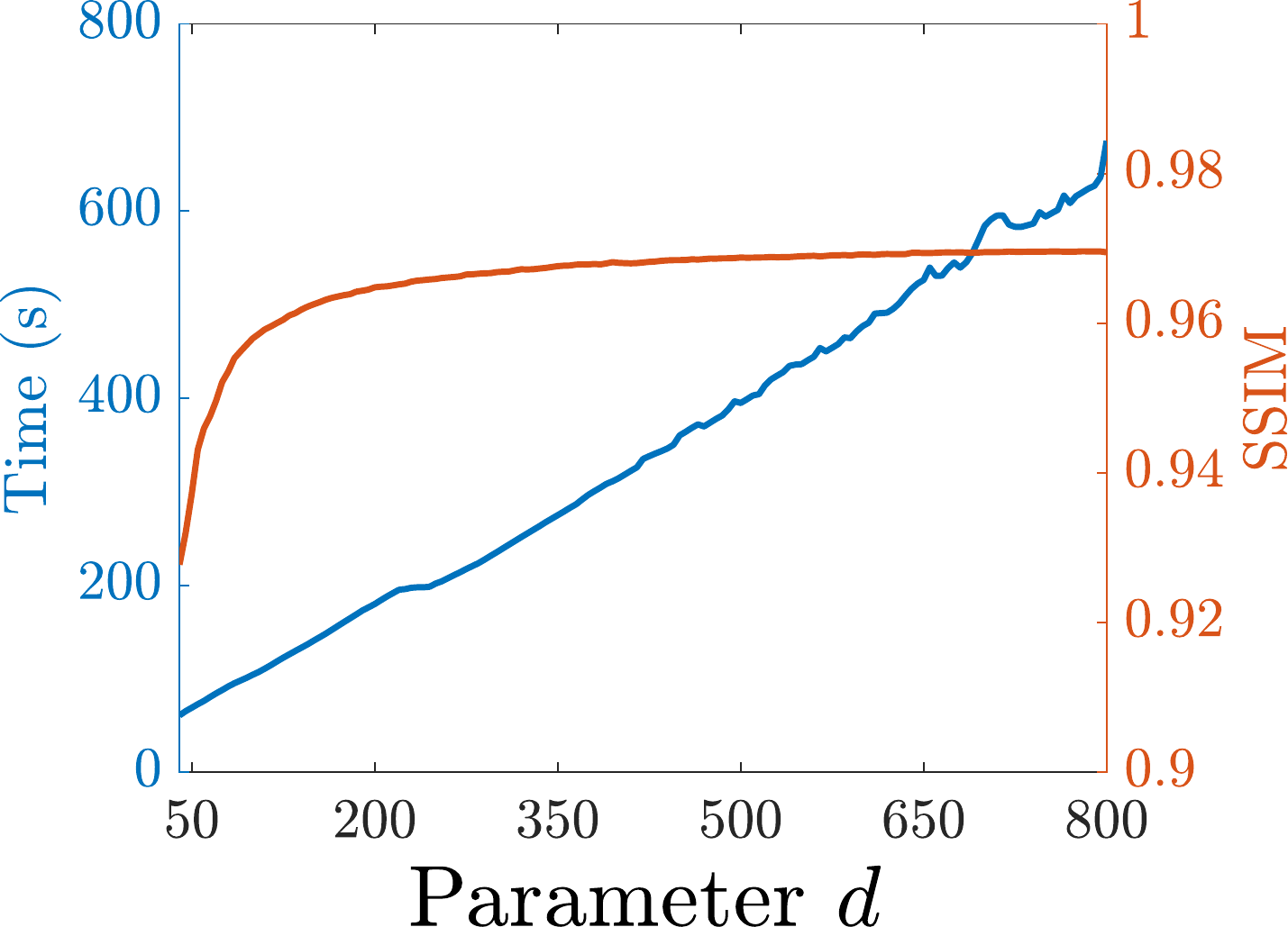} &\includegraphics[width=0.22\linewidth]{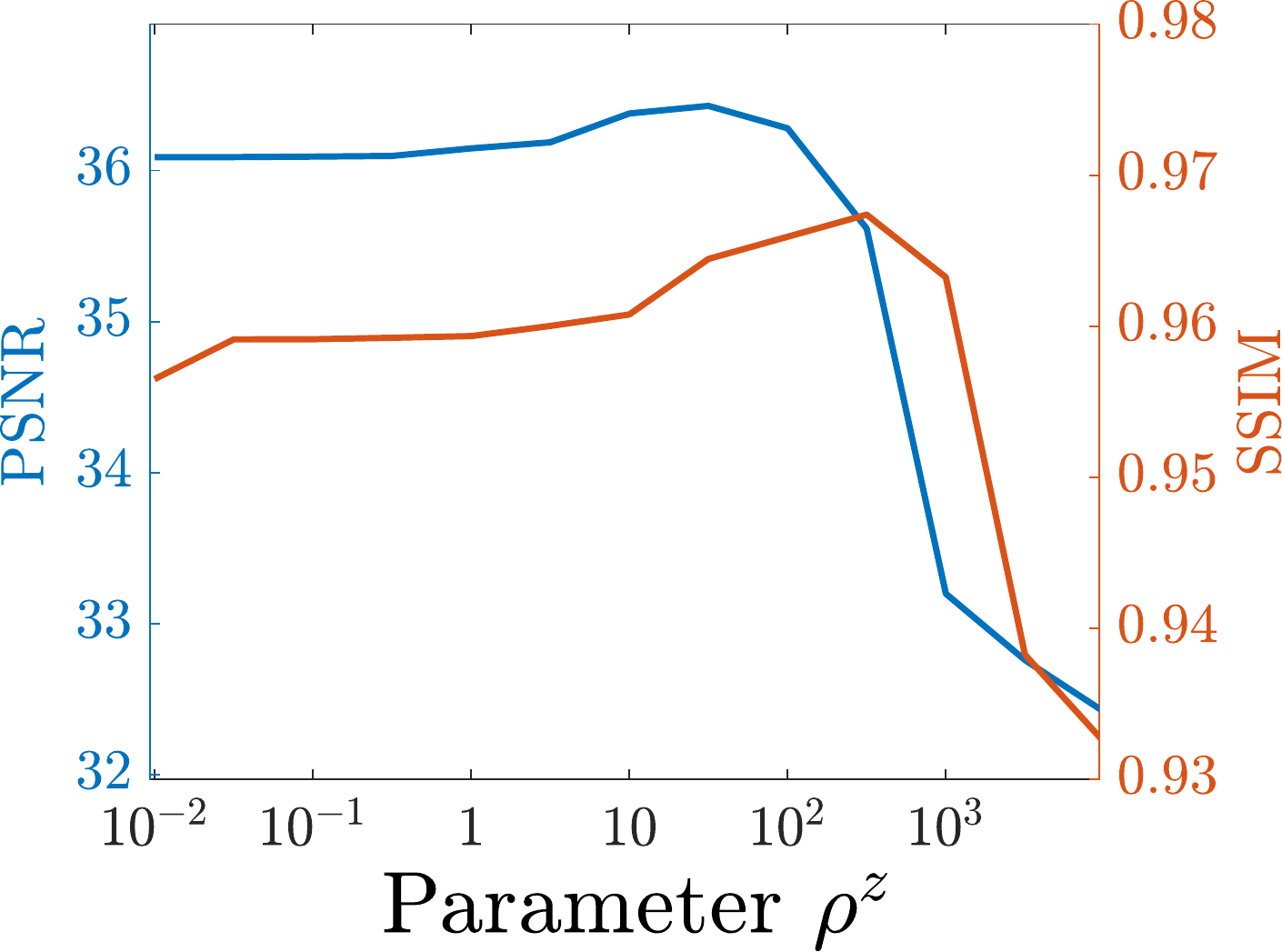} &\includegraphics[width=0.22\linewidth]{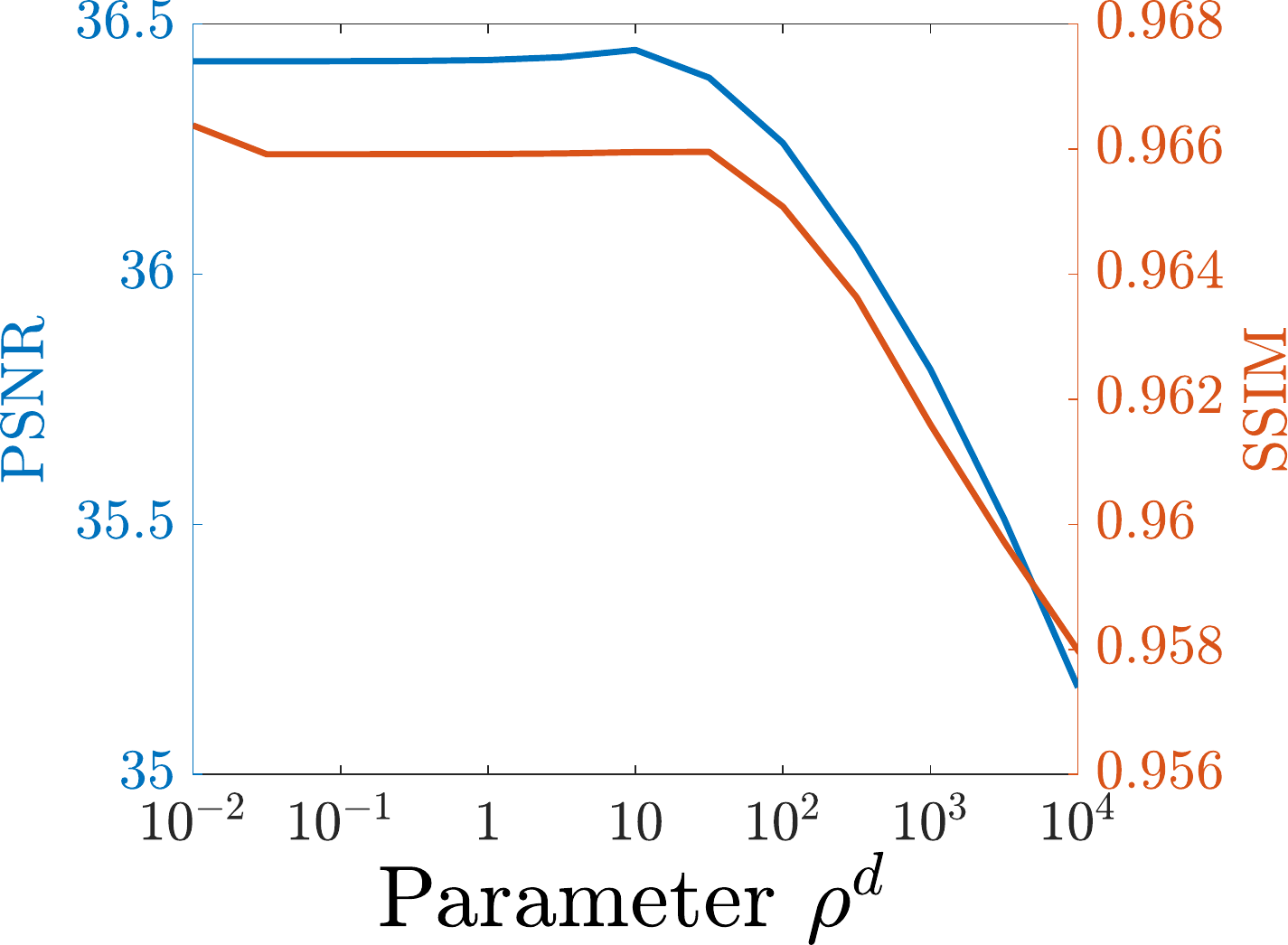}&\includegraphics[width=0.22\linewidth]{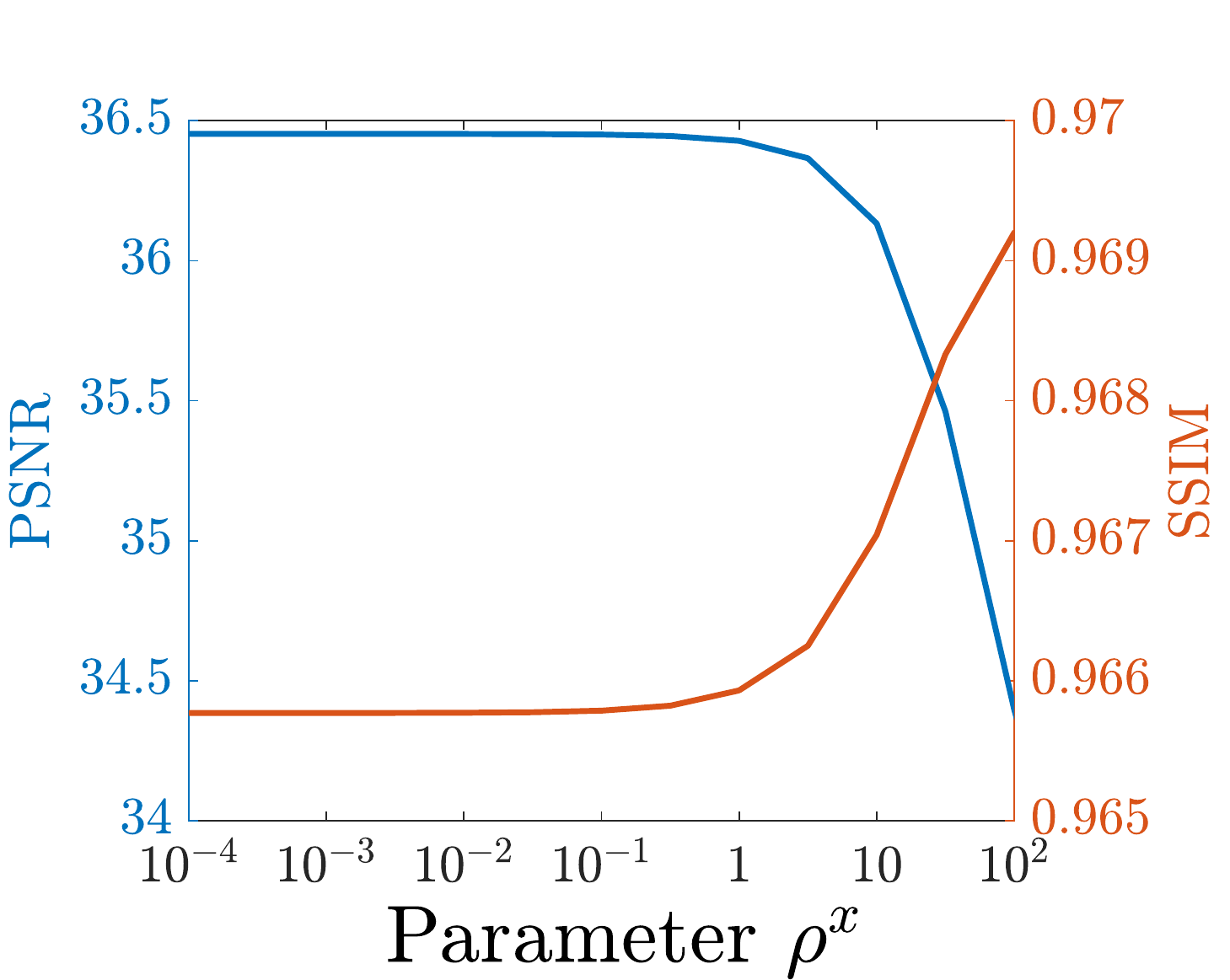}
\end{tabular}
\caption{{\red The running time in seconds and SSIM values with different numbers of dictionary atoms ($d$), and} PSNR and SSIM values of the result by our method with different $\rho^z$, $\rho^d$, and $\rho^x$, on the video ``\textit{foreman}'' (SR = 50\%).}
\label{fig-Para}
\end{figure*}

In this section, all the methods are conducted on the MRI data\footnote{\scriptsize  \url{https://brainweb.bic.mni.mcgill.ca/brainweb/selection_normal.html}} of the size $142\times178\times121$.
This MRI data provides a 3D view of the brain part of a human being. That is, all the modes of this MRI data are corresponding to spatial information.
The sampling rates are set from 10\% to 50\%. Similar to the video data, we compute the mean values of PSNR, SSIM, and UIQI of each frontal slices and report them in Tab. \ref{table-MRI}. From Tab. \ref{table-MRI}, we can find that DTNN outperforms compared methods while DCTNN and FTNN alternatively obtain the second best values. Fig. \ref{fig-MRI} presents the 61-th frontal slice of the results by different methods. For the enlarged white manner area, which is smooth, the results by our DTNN is the cleanest compared with the results by other methods.

\subsection{Discussions}
\begin{table}[!t]
\renewcommand\arraystretch{0.9}\setlength{\tabcolsep}{1.5pt}
\centering\red
\caption{\red PSNR and SSIM values of results by DTNN and traditional dictionary learning method with different sampling rates on the \textbf{video} data ``\textit{foreman}''.}
\begin{tabular}{cccccccccccccccccccccccc}\toprule
SR              & \multicolumn{2}{c}{10\%} & \multicolumn{2}{c}{20\%} & \multicolumn{2}{c}{30\%} & \multicolumn{2}{c}{40\%}& \multicolumn{2}{c}{50\%} \\\cmidrule{1-11}
Method          &  PSNR  & SSIM   &  PSNR  & SSIM   &  PSNR  & SSIM   &  PSNR  & SSIM   &  PSNR  & SSIM\\\midrule
    $\ell_1$    & 23.80     & 0.703  & 25.08  & 0.751  & 29.81  & 0.909  & 31.46  & 0.918  & 33.29  & 0.942  \\
$\ell_{1,2}$    & 23.86     & 0.708  & 25.23  & 0.760  & 30.20  & 0.911  & 31.76  & 0.944  & 33.36  & 0.959  \\
DTNN       & \bf 26.22 & \bf 0.799  & \bf 29.26  & \bf 0.875  & \bf 31.77  & \bf 0.917  & \bf 34.13  & \bf 0.946  & \bf 36.32  & \bf 0.964  \\
\bottomrule
    \end{tabular}%
\label{Tab-Sparse}
\end{table}

\begin{table}[!t]
\red
\renewcommand\arraystretch{0.5}\setlength{\tabcolsep}{3pt}
\centering
\caption{PSNR, SSIM, and UIQI values of results by our method with different initializations on the video data \textit{foreman} ($\text{SR}=50\%$). The {\bf best} and the \underline{second best} values are respectively highlighted by {\bf boldface} and \underline{underline}.}
\begin{tabular}{ccccccc}\toprule
\multicolumn{3}{c}{Method} & & \footnotesize PSNR  &\footnotesize SSIM &\footnotesize UIQI \\\midrule
\multicolumn{3}{c}{Observed}& &            6.51   &            0.047   &            0.056\\
\multicolumn{3}{c}{HaLRTC}& &            31.85   &            0.928   &            0.853\\
\multicolumn{3}{c}{DCTNN}& &  34.07   &  0.934   &  0.854\\\cmidrule{1-7}
\multirow{18}{*}{DTNN}& $\mathcal{X}_0$ & $\mathbf D_0$\\\cmidrule{2-7}
& \multirow{2}{*}{Random} & Random && 14.65  & 0.175 & 0.197\\
&  & Tubes&& 13.29  & 0.133 & 0.153\\\cmidrule{2-7}
& \multirow{2}{*}{Interpolation*} & Random&& 33.52  & 0.939 & 0.862\\
&  & Tubes* && 36.32  & \underline{0.964} & \underline{0.907}\\\cmidrule{2-7}
& \multirow{2}{*}{HaLRTC} & Random&& 32.44  & 0.926 & 0.843\\
&  & Tubes& & \underline{36.50}  & \bf 0.965 & \bf 0.908\\\cmidrule{2-7}
& \multirow{2}{*}{DCTNN} & Random&& 32.74  & 0.930 & 0.849\\
&  & Tubes&& \bf 36.51  &\bf  0.965 & \bf 0.908\\\bottomrule
\multicolumn{6}{l}{*Default setting in our experiments.}
\end{tabular}%
\label{Tab-ini}
\end{table}

{\red
\subsubsection{Comparisons with traditional dictionary learning approaches}\label{Sec-SvsLR}
In this part, we compare our method with traditional dictionary methods.
First, the data $\mathcal{O}\in\mathbb{R}^{n_1\times n_2\times n_3}$ and the coefficient $\mathcal{Z}\in\mathbb{R}^{n_1\times n_2\times d}$ in the tensor format are reshaped into the matrix form via the $\tt{unfold}_3$ operation, i.e., $\mathbf O_{(3)} \in\mathbb{R}^{n_3\times n_1n_2}= \tt{unfold}_3(\mathcal{O})$ and $\mathbf Z_{(3)}\in\mathbb{R}^{d\times n_1n_2} = \tt{unfold}_3(\mathcal{Z})$.
The representation formulary also turns from $\mathcal{O}\approx\mathcal{Z}\times_3\mathbf D$ to $\mathbf O_{(3)}\approx\mathbf D\mathbf Z_{(3)}$.
The tubes of $\mathcal{Z}$  constitute the columns of $\mathbf Z_{(3)}$,  the $i$-th row of $\mathbf Z_{(3)}$ is reshaped from the $i$-th frontal slice of $\mathcal{Z}$.
Then, for the coding coefficient matrix $\mathbf Z_{(3)}$ (also denoted as $\mathbf Z$ for convenience), we regularize it with the common $\|\mathbf Z\|_1= \sum_{ij}|Z_{(3)}|$, which is usually used to enhance the sparsity, and the $\|\mathbf Z\|_{1,2}=\sum_j\sqrt{\sum_iZ_{ij}^2}$, which could exploit the group sparsity of the columns.
For a fair comparison, the algorithm with theoretical guaranteed convergency in \cite{bao2015dictionary} is adopt to optimize these two models.
The video ``\textit{foreman}'' is selected with sampling rates varying from $10\%$ to $50\%$ for testing.
We exhibit the PSNR and SSIM values of the results in Table \ref{Tab-Sparse}.
These two traditional dictionary learning methods are respectively denoted as $\ell_1$ and $\ell_{1,2}$.
We can see from Table \ref{Tab-Sparse} that the $\ell_{1,2}$ constraint is also effective when the sampling rate is bigger than $30\%$ (See the SSIM values). However, when the sampling rate is low, the margin between $\ell_{1,2}$ and our DTNN becomes much larger. Therefore, we can deduce that using the combination of dictionary's atoms (the low-rank constraint) would be helpful to eliminate the deviation caused by inaccurate estimation of the dictionary from incomplete data.
This also supports our statement at the end of Sec. \ref{Sec:Model1} that we need both the learned dictionary and the specific low-rank structure of the coefficients for the accurate completion of the data.}

{\red
\subsubsection{Parameters}\label{Sec:Para}
In our experiments, we find that four parameters mainly affect the performance of our method, i.e., the number of the dictionary atoms $d$ and the proximal parameters $\rho^z$, $\rho^d$, and $\rho^x$. Although $\rho^z$, $\rho^d$, and $\rho^x$ can be finely specified for each iteration, we respectively fix their values across our algorithm to reduce the parameter tuning burden. To test the effects from different values of them, we conduct experiments on the video ``\textit{foreman}'' with setting the sampling rate as $50\%$.

When testing one parameter, other three are fixed as default values. As for $d$, we vary its value from 40 ($0.8n_3$) to 800 ($16n_3$) with a step size $5$. $\rho^z$ and $\rho^d$ are tested with candidates $\{10^{-2},10^{-1.5},\cdots,10^4\}$ while $\rho^x$ varies from $10^4$ to $10^2$.
We illustrate the running time and SSIM values with respect to different values of $d$ and the PSNR and SSIM values with respect to different values of $\rho^z$, $\rho^d$, and $\rho^x$  in Fig. \ref{fig-Para}.
From Fig. \ref{fig-Para}, we can see that as $d$ increases, the performance of our method becomes better while the running time also grows. Our default setting ($d=5n_3=250$) is a compromise between the effectiveness and efficiency. Meanwhile, we can see the performance of our method is more sensitive to $\rho^z$ and our method could obtain satisfactory results with a wide range of $\rho^d$ and $\rho^x$.
}

\begin{figure}[!t]
\setlength{\tabcolsep}{5pt}
\renewcommand\arraystretch{0.6}
\centering
\includegraphics[width=0.90\linewidth]{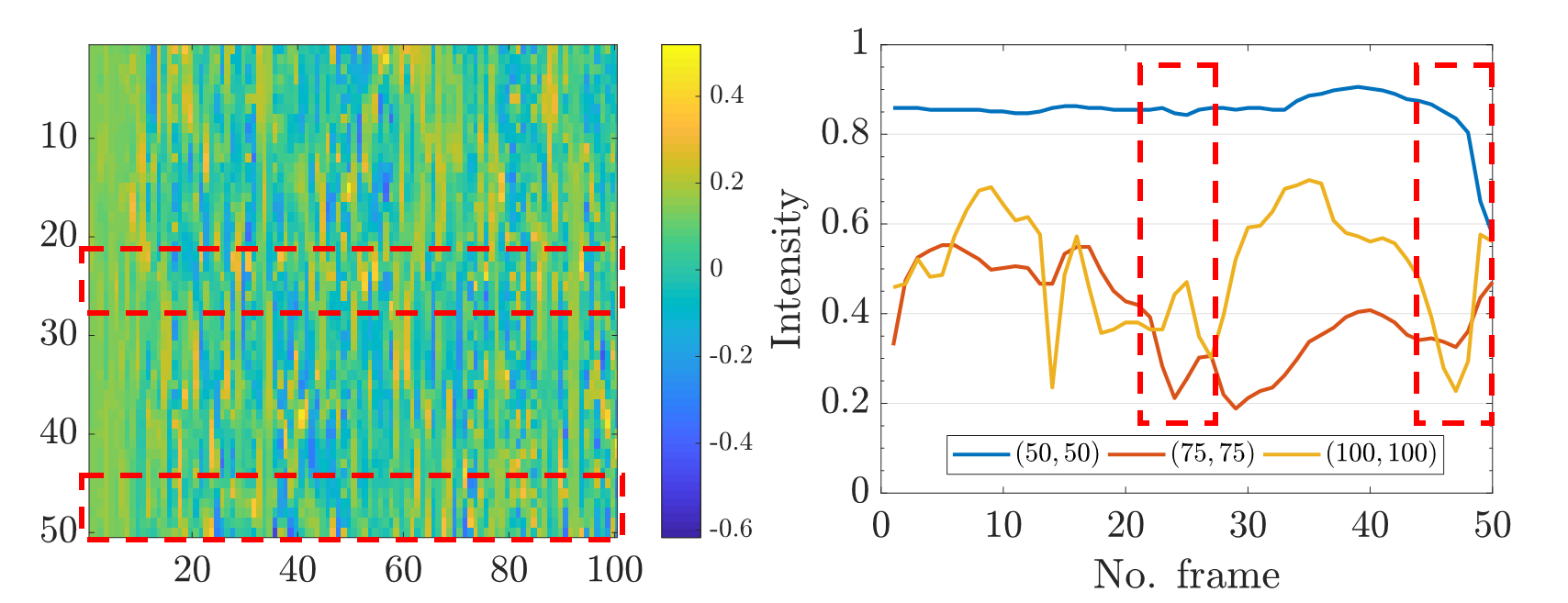}
\includegraphics[width=0.99\linewidth]{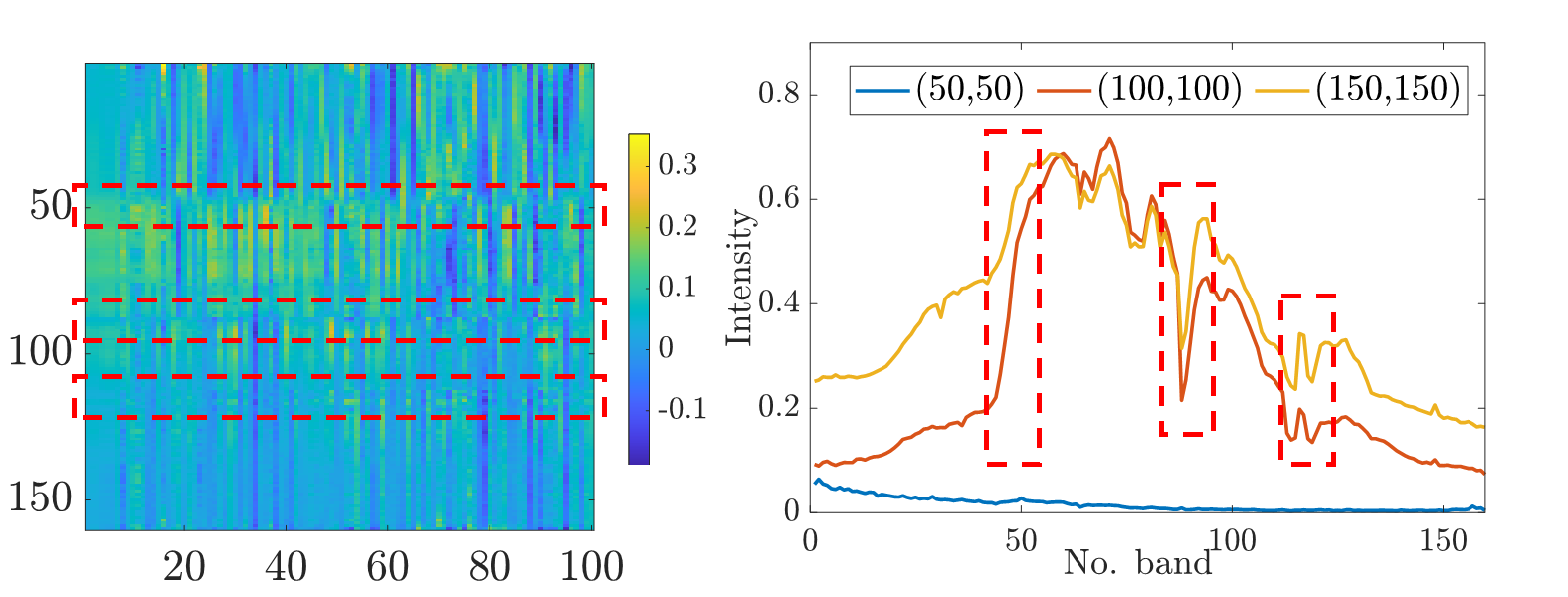}
\caption{\red The learned dictionaries (left) and the tubes of the original data (right). Top: the video ``\textit{foreman}'' of the size $144\times176\times50$ with SR=50\%. Bottom: the HSI Washinton DC Mall of the size $256\times256\times160$ with SR=40\%.}
\label{dict}
\end{figure}

\begin{figure}[!t]
\setlength{\tabcolsep}{5pt}
\renewcommand\arraystretch{0.6}
\centering
\includegraphics[width=0.80\linewidth]{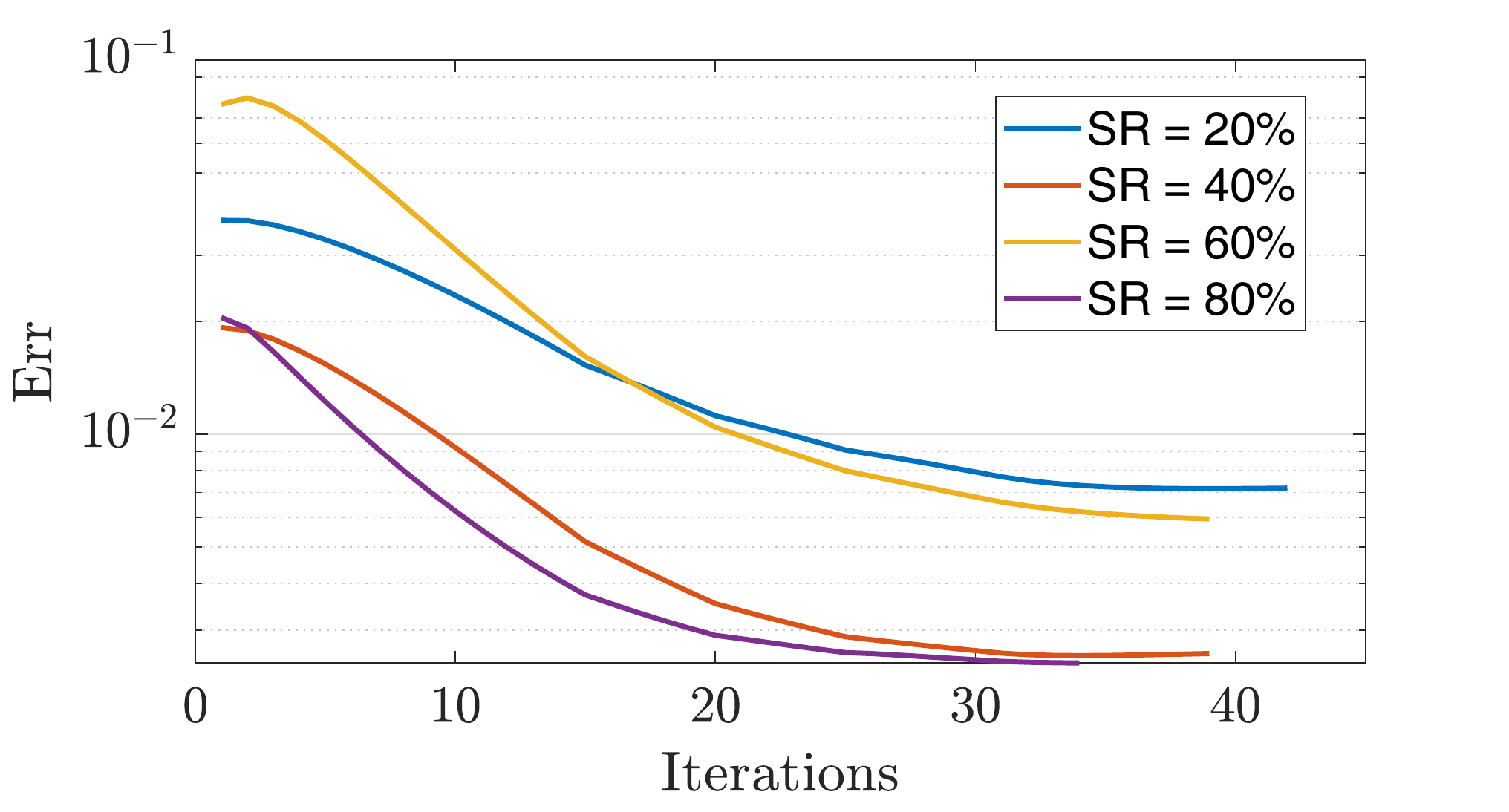}
\caption{\red The estimation error  (Err) of the dictionary with respect to iterations for the synthetic data.}
\label{dict2}
\end{figure}

{\red
\subsubsection{Initializations}\label{2121}
In this part, we test different initialization strategies. Other than the default setting, we employ the random tensor, whose values are uniformly distributed in the interval $[0,1]$, and the results from HaLRTC and DCTNN, which are fast, as initial guesses of $\mathcal{X}_0$. Meanwhile, we also initialize the dictionary using random values following a standard normal distribution. Also, we take the video ``\textit{foreman}'' with sampling rate $50\%$ as an example. The results are shown in Table \ref{Tab-ini}.

From Table \ref{Tab-ini}, we can see that when $\mathcal{X}_0$ is randomly initialized, the performance of our method is poor. When implementing our method with $\mathcal{X}_0$s using the results from HaLRTC and DCTNN, the performances are better than the default setting. This shows that our method indeed relies on the initialization as many nonconvex optimization methods. As for $\mathbf D_0$, our method performs well when using tubes of $\mathcal{X}_0$, as it would contributes to flexibility of $\mathbf D$.}

\subsubsection{Learned dictionaries}
In Fig. \ref{dict}, we exhibit the fist 100 columns of the learned dictionaries together with the plotting of three tubes of the original data. From the red boxes with dashed line, we can see that when the tubes, i.e., the vectors along the third dimension, of the original data fluctuate, the corresponding areas of the dictionaries' atoms (columns) tend not to be smooth. 
This reflects that the dictionaries learned by our method is flexible and adaptive to different types of data.
\begin{figure}[!t]
\setlength{\tabcolsep}{5pt}
\renewcommand\arraystretch{0.6}
\centering
\includegraphics[width=0.99\linewidth]{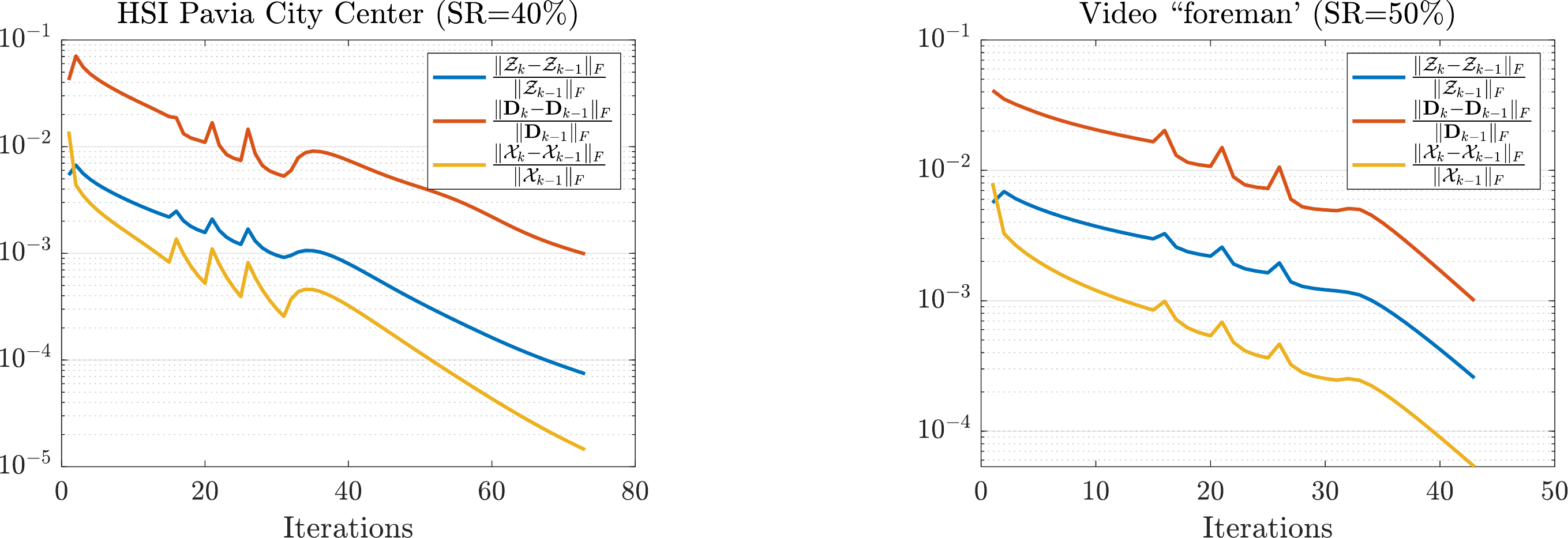}
\caption{The relative changes of the variables. Left: MRI data with SR=30\%. Right: video data ``foreman'' with SR=50\%.}
\label{conver}
\end{figure}
\begin{table}[!t]
\red
\renewcommand\arraystretch{0.7}\setlength{\tabcolsep}{2pt}
\centering
\caption{The computational complexity of each method to deal with a tensor with the size $n_1\times n_2\times n_3$, and the  averaged iterations needed for different types of data.}
\begin{tabular}{ccccccc}\toprule
\multirow{3}{*}{Method}& \multirow{2}{*}{Complexity} & \multicolumn{4}{c}{Iterations} \\\cmidrule{3-6}
& per iteration                                 &Video  & HSI   & MRI   & Traffic\\\midrule
HaLRTC& O$(n_1n_2n_3\sum_{j=1}^3n_j)$           &  59   &97     & 52    &108\\
BCPF & O$(3R^2|\Omega|+R^3)$                    &14     &14     & 12    &15\\
TRLRF & O$(R^2n_1n_2n_3+R^6)$                   &487    &500    & 493   &500\\
TNN  & O$(n_1n_2n_3(\log n_3 +\min(n_1,n_2))$   &86     &77     & 76    &94\\
DCTNN  & O$(n_1n_2n_3(\log n_3 +n_1)$           &97     &91     & 88    &103\\
FTNN & O$(\omega n_1n_2n_3(n_3 +\min(n_1,n_2))$ &86     &79     & 40    &68\\
DTNN & O$(dn_1n_2(dn_3+\min(n_1,n_2)+n_3)$      &61     &72     & 62    &52\\
\bottomrule
\end{tabular}%
\label{Tab-complexiy}
\end{table}
{\red
Meanwhile, we simulate a tensor $\mathcal{X}\in\mathbb{R}^{50\times 50\times 50}=\mathcal{Z}\times_3\mathbf D$, where frontal slices of $\mathcal{Z}\in\mathbb{R}^{50\times 50\times 250}$ is obtained via the multiplication between randomly generated matrices of sizes $50\times 5$ and $5\times 50$, and $\mathbf D\in\mathbb{R}^{50\times 250}$ is randomly generated and normalized with the norm of its columns equaling to 1. Thus, we have the ground-truth of the dictionary and we adopt the estimation error (Err)\footnote{\red The estimation error is defined as
$
\text{Err} = \frac{1}{d}\sum_{i = 1}^d(1-|(\mathbf d^i_\text{Est})^\top\mathbf d^{i_0}_\text{GT} |)
$,
where $\mathbf d^i_\text{Est}$ is the $i$-the atom (column) of the estimated dictionary, $\mathbf d^{i_0}_\text{GT} $ is the $i_0$-th atom of the ground-truth dictionary, and $i_0=\arg\max_{j\in\mathbb{N}^+,1\leq j\leq d,j\neq i} 
|(\mathbf d^i_\text{Est})^\top \mathbf d^{j}_\text{GT} |$.
} of the dictionary \cite{yu2020dictionary}
as a quantitative metric to measure the accuracy of the estimated dictionary. We plot the estimation errors with respect to iteration numbers in the Fig. \ref{dict2}. Although the initial Err values are different owing to the initilization stage, we can see that the Err is becoming smaller as iteration goes on. This shows that our method could enforce the estimated dictionary being close to the ground-truth under different sampling rates.}


\subsubsection{Convergency behaviors}
When the largest relative change of the variables, i.e., $\max\{ \frac{\|\mathcal{Z}_k-\mathcal{Z}_{k-1}\|_F}{\|\mathcal{Z}_{k-1}\|_F}, \frac{\|\mathbf{D}_k-\mathbf{D}_{k-1}\|_F}{\|\mathbf{D}_{k-1}\|_F}, \frac{\|\mathcal{X}_k-\mathcal{X}_{k-1}\|_F}{\|\mathcal{X}_{k-1}\|_F} \}$, is smaller than $10^{-3}$, we consider that our algorithm converges and stop the iterations.
In Fig. \ref{conver}, we present the relative changes with respect to the iterations in our experiments on the HSI data Pavia City Center and the video data ``\textit{foreman}''. Three obvious fluctuations in each curve are is in accord with our parameter setting of enlarging $\rho$ at the 15-th, 20-th, and 25-th iterations. The overall downward trend of the curves in Fig. \ref{conver} illustrates that our method converges quickly.

{\red Moreover, we list the computational complexity of compared methods and the iterations needed for different types of data in Table \ref{Tab-complexiy}. The CP-rank used in BCPF is $R$ and the TR-rank employed in TRLRF is $[R,R,R]$. For FTNN, $\omega$ corresponds to the construction of the framelet system. Although, our method generally needs fewer iterations than other TNN induced methods (TNN, DCTNN, and FTNN), it costs more running time. The main reason is that the computation complexity of our method is high as $d$ is much bigger than $n_3$.
}
\section{Conclusions}\label{Sec-Con}

In this paper, we have introduced the data-adaptive dictionary and low-rank coding for third-order tensor completion. In the completion model, we have proposed to minimize the low-rankness of each tensor slice containing the coding coefficients.
To optimize this model, we design a multi-block proximal alternating minimization algorithm, the sequence generated by which would globally converge to a critical point. Numerical experiments conducted on various types of real-world data show the {\red superiority} of the proposed method.

As a future research work, we will consider how to use the proposed model and idea
to analyze and study
a tensor-based representation learning method for multi-view clustering. Here multi-view data
as a third-order tensor expresses each tensorial data point as a low rank representation of the
learned dictionary basis.

\section*{Acknowledgment}
The authors would like to thank the authors of \cite{Liu2013PAMItensor,yuan2019tensor,zhao2015bayesian,zhang2017exact,tproduct2018lu} for their generous sharing of their codes.
The authors would like to thank the support from Financial Intelligence and Financial Engineering Research Key Laboratory of Sichuan province.

{\footnotesize
\bibliographystyle{ieeetran}
\bibliography{ref}

\begin{thebibliography}{10}
\providecommand{\url}[1]{#1}
\csname url@samestyle\endcsname
\providecommand{\newblock}{\relax}
\providecommand{\bibinfo}[2]{#2}
\providecommand{\BIBentrySTDinterwordspacing}{\spaceskip=0pt\relax}
\providecommand{\BIBentryALTinterwordstretchfactor}{4}
\providecommand{\BIBentryALTinterwordspacing}{\spaceskip=\fontdimen2\font plus
\BIBentryALTinterwordstretchfactor\fontdimen3\font minus
  \fontdimen4\font\relax}
\providecommand{\BIBforeignlanguage}[2]{{%
\expandafter\ifx\csname l@#1\endcsname\relax
\typeout{** WARNING: IEEEtran.bst: No hyphenation pattern has been}%
\typeout{** loaded for the language `#1'. Using the pattern for}%
\typeout{** the default language instead.}%
\else
\language=\csname l@#1\endcsname
\fi
#2}}
\providecommand{\BIBdecl}{\relax}
\BIBdecl

\bibitem{Bertalmio2000imInpait}
M.~Bertalmio, G.~Sapiro, V.~Caselles, and C.~Ballester, ``Image inpainting,''
  in \emph{the Annual Conference on Computer Graphics and Interactive
  Techniques (SIGGRAPH)}, 2000, pp. 417--424.

\bibitem{Komodakis2006Global}
N.~Komodakis, ``Image completion using global optimization,'' in \emph{the IEEE
  Conference on Computer Vision and Pattern Recognition (CVPR)}, 2006, pp.
  442--452.

\bibitem{Korah2007TIP}
T.~Korah and C.~Rasmussen, ``Spatiotemporal inpainting for recovering texture
  maps of occluded building facades,'' \emph{IEEE Transactions on Image
  Processing}, vol.~16, no.~9, pp. 2262--2271, 2007.

\bibitem{Liu2013PAMItensor}
J.~Liu, P.~Musialski, P.~Wonka, and J.~Ye, ``Tensor completion for estimating
  missing values in visual data,'' \emph{IEEE Transactions on Pattern Analysis
  and Machine Intelligence}, vol.~35, no.~1, pp. 208--220, 2013.

\bibitem{zhao2020deep}
X.-L. Zhao, W.-H. Xu, T.-X. Jiang, Y.~Wang, and M.~K. Ng, ``Deep plug-and-play
  prior for low-rank tensor completion,'' \emph{Neurocomputing}, vol. 400, pp.
  137--149, 2020.

\bibitem{koller2015high}
R.~Koller, L.~Schmid, N.~Matsuda, T.~Niederberger, L.~Spinoulas, O.~Cossairt,
  G.~Schuster, and A.~Katsaggelos, ``High spatio-temporal resolution video with
  compressed sensing.'' \emph{Optics express}, vol.~23, no.~12, p. 15992, 2015.

\bibitem{MRITV}
V.~N. Varghees, M.~S. Manikandan, and R.~Gini, ``Adaptive {MRI} image denoising
  using total-variation and local noise estimation,'' in \emph{the
  International Conference on Advances in Engineering, Science and Management
  (ICAESM)}, 2012, pp. 506--511.

\bibitem{zhuang2018fast}
L.~Zhuang and J.~M. Bioucas-Dias, ``Fast hyperspectral image denoising and
  inpainting based on low-rank and sparse representations,'' \emph{IEEE Journal
  of Selected Topics in Applied Earth Observations and Remote Sensing},
  vol.~11, no.~3, pp. 730--742, 2018.

\bibitem{bioucas2013hyperspectral}
J.~M. Bioucas-Dias, A.~Plaza, G.~Camps-Valls, P.~Scheunders, N.~Nasrabadi, and
  J.~Chanussot, ``Hyperspectral remote sensing data analysis and future
  challenges,'' \emph{IEEE Geoscience and remote sensing magazine}, vol.~1,
  no.~2, pp. 6--36, 2013.

\bibitem{kiers2000towards}
H.~A. Kiers, ``Towards a standardized notation and terminology in multiway
  analysis,'' \emph{Journal of Chemometrics: A Journal of the Chemometrics
  Society}, vol.~14, no.~3, pp. 105--122, 2000.

\bibitem{hillar2013most}
C.~J. Hillar and L.-H. Lim, ``Most tensor problems are {NP}-hard,''
  \emph{Journal of the ACM}, vol.~60, no.~6, pp. 1--39, 2013.

\bibitem{zhao2015bayesian}
Q.~Zhao, L.~Zhang, and A.~Cichocki, ``Bayesian {CP} factorization of incomplete
  tensors with automatic rank determination,'' \emph{IEEE Transactions on
  Pattern Analysis and Machine Intelligence}, vol.~37, no.~9, pp. 1751--1763,
  2015.

\bibitem{han2018generalized}
Z.~Han, Y.~Wang, Q.~Zhao, D.~Meng, L.~Lin, Y.~Tang \emph{et~al.}, ``A
  generalized model for robust tensor factorization with noise modeling by
  mixture of gaussians,'' \emph{IEEE Transactions on Neural Networks and
  Learning Systems}, vol.~29, no.~11, pp. 5380--5393, 2018.

\bibitem{zhao2016bayesian}
Q.~Zhao, G.~Zhou, L.~Zhang, A.~Cichocki, and S.-I. Amari, ``Bayesian robust
  tensor factorization for incomplete multiway data,'' \emph{IEEE Transactions
  on Neural Networks and Learning Systems}, vol.~27, no.~4, pp. 736--748, 2016.

\bibitem{tucker1966some}
L.~R. Tucker, ``Some mathematical notes on three-mode factor analysis,''
  \emph{Psychometrika}, vol.~31, no.~3, pp. 279--311, 1966.

\bibitem{zhang2018nonconvex}
X.~Zhang, ``A nonconvex relaxation approach to low-rank tensor completion,''
  \emph{IEEE Transactions on Neural Networks and Learning Systems}, vol.~30,
  no.~6, pp. 1659--1671, 2018.

\bibitem{oseledets2011tensor}
I.~V. Oseledets, ``Tensor-train decomposition,'' \emph{SIAM Journal on
  Scientific Computing}, vol.~33, no.~5, pp. 2295--2317, 2011.

\bibitem{bengua2017efficient}
J.~A. Bengua, H.~N. Phien, H.~D. Tuan, and M.~N. Do, ``Efficient tensor
  completion for color image and video recovery: Low-rank tensor train,''
  \emph{IEEE Transactions on Image Processing}, vol.~26, no.~5, pp. 2466--2479,
  2017.

\bibitem{dian2019learning}
R.~Dian, S.~Li, and L.~Fang, ``Learning a low tensor-train rank representation
  for hyperspectral image super-resolution,'' \emph{IEEE Transactions on Neural
  Networks and Learning Systems}, vol.~30, no.~9, pp. 2672--2683, 2019.

\bibitem{liu2020low}
Y.~Liu, J.~Liu, and C.~Zhu, ``Low-rank tensor train coefficient array
  estimation for tensor-on-tensor regression,'' \emph{IEEE Transactions on
  Neural Networks and Learning Systems}, 2020.

\bibitem{zhao2016tensor}
Q.~Zhao, G.~Zhou, S.~Xie, L.~Zhang, and A.~Cichocki, ``Tensor ring
  decomposition,'' \emph{arXiv preprint arXiv:1606.05535}, 2016.

\bibitem{yuan2019tensor}
L.~Yuan, C.~Li, D.~Mandic, J.~Cao, and Q.~Zhao, ``Tensor ring decomposition
  with rank minimization on latent space: An efficient approach for tensor
  completion,'' in \emph{Proceedings of the AAAI Conference on Artificial
  Intelligence}, vol.~33, 2019, pp. 9151--9158.

\bibitem{yu2020low}
J.~Yu, G.~Zhou, C.~Li, Q.~Zhao, and S.~Xie, ``Low tensor-ring rank completion
  by parallel matrix factorization,'' \emph{IEEE Transactions on Neural
  Networks and Learning Systems}, 2020~({DOI}:10.1109/TNNLS.2020.3009210).

\bibitem{long2019low}
Z.~Long, Y.~Liu, L.~Chen, and C.~Zhu, ``Low rank tensor completion for multiway
  visual data,'' \emph{Signal Processing}, vol. 155, pp. 301--316, 2019.

\bibitem{song2019tensor}
Q.~Song, H.~Ge, J.~Caverlee, and X.~Hu, ``Tensor completion algorithms in big
  data analytics,'' \emph{ACM Transactions on Knowledge Discovery from Data
  (TKDD)}, vol.~13, no.~1, p.~6, 2019.

\bibitem{braman2010third}
K.~Braman, ``Third-order tensors as linear operators on a space of matrices,''
  \emph{Linear Algebra and its Applications}, vol. 433, no.~7, pp. 1241--1253,
  2010.

\bibitem{kilmer2011factorization}
M.~E. Kilmer and C.~D. Martin, ``Factorization strategies for third-order
  tensors,'' \emph{Linear Algebra and its Applications}, vol. 435, no.~3, pp.
  641--658, 2011.

\bibitem{hao2013facial}
N.~Hao, M.~E. Kilmer, K.~Braman, and R.~C. Hoover, ``Facial recognition using
  tensor-tensor decompositions,'' \emph{SIAM Journal on Imaging Sciences},
  vol.~6, no.~1, pp. 437--463, 2013.

\bibitem{kilmer2013third}
M.~E. Kilmer, K.~Braman, N.~Hao, and R.~C. Hoover, ``Third-order tensors as
  operators on matrices: A theoretical and computational framework with
  applications in imaging,'' \emph{SIAM Journal on Matrix Analysis and
  Applications}, vol.~34, no.~1, pp. 148--172, 2013.

\bibitem{zhang2014novel}
Z.~Zhang, G.~Ely, S.~Aeron, N.~Hao, and M.~Kilmer, ``Novel methods for
  multilinear data completion and de-noising based on tensor-{SVD},'' in
  \emph{the IEEE Conference on Computer Vision and Pattern Recognition (CVPR)},
  2014, pp. 3842--3849.

\bibitem{zhang2017exact}
Z.~Zhang and S.~Aeron, ``Exact tensor completion using t-{SVD},'' \emph{IEEE
  Transactions on Signal Processing}, vol.~65, no.~6, pp. 1511--1526, 2017.

\bibitem{jiang2019robust}
J.~Q. Jiang and M.~K. Ng, ``Robust low-tubal-rank tensor completion via convex
  optimization,'' in \emph{Proceedings of the 28th International Joint
  Conference on Artificial Intelligence}, 2019, pp. 2649--2655.

\bibitem{wang2019robust}
A.~Wang, X.~Song, X.~Wu, Z.~Lai, and Z.~Jin, ``Robust low-tubal-rank tensor
  completion,'' in \emph{IEEE International Conference on Acoustics, Speech and
  Signal Processing (ICASSP)}.\hskip 1em plus 0.5em minus 0.4em\relax IEEE,
  2019, pp. 3432--3436.

\bibitem{kernfeld2015tensor}
E.~Kernfeld, M.~Kilmer, and S.~Aeron, ``Tensor--tensor products with invertible
  linear transforms,'' \emph{Linear Algebra and its Applications}, vol. 485,
  pp. 545--570, 2015.

\bibitem{lu2019low}
C.~Lu, X.~Peng, and Y.~Wei, ``Low-rank tensor completion with a new tensor
  nuclear norm induced by invertible linear transforms,'' in \emph{Proceedings
  of the IEEE Conference on Computer Vision and Pattern Recognition}, 2019, pp.
  5996--6004.

\bibitem{xu2019fast}
W.-H. Xu, X.-L. Zhao, and M.~Ng, ``A fast algorithm for cosine transform based
  tensor singular value decomposition,'' \emph{arXiv preprint
  arXiv:1902.03070}, 2019.

\bibitem{song2019robust}
G.~Song, M.~K. Ng, and X.~Zhang, ``Robust tensor completion using transformed
  tensor singular value decomposition,'' \emph{Numerical Linear Algebra with
  Applications}, vol.~27, no.~3, p. e2299, 2020.

\bibitem{jiang2019framelet}
T.-X. Jiang, M.~K. Ng, X.-L. Zhao, and T.-Z. Huang, ``Framelet representation
  of tensor nuclear norm for third-order tensor completion,'' \emph{IEEE
  Transactions on Image Processing}, vol.~29, pp. 7233--7244, 2020.

\bibitem{zheng2020tensor}
Y.-B. Zheng, T.-Z. Huang, X.-L. Zhao, T.-X. Jiang, T.-Y. Ji, and T.-H. Ma,
  ``Tensor {N}-tubal rank and its convex relaxation for low-rank tensor
  recovery,'' \emph{Information Sciences}, vol. 532, pp. 170--189, 2020.

\bibitem{martin2013order}
C.~D. Martin, R.~Shafer, and B.~LaRue, ``An order-$p$ tensor factorization with
  applications in imaging,'' \emph{SIAM Journal on Scientific Computing},
  vol.~35, no.~1, pp. A474--A490, 2013.

\bibitem{lu2019tensortpami}
C.~Lu, J.~Feng, Y.~Chen, W.~Liu, Z.~Lin, and S.~Yan, ``Tensor robust principal
  component analysis with a new tensor nuclear norm,'' \emph{IEEE transactions
  on pattern analysis and machine intelligence}, vol.~42, no.~4, pp. 925--938,
  2019.

\bibitem{liu2021hyperspectral}
Y.-Y. Liu, X.-L. Zhao, Y.-B. Zheng, T.-H. Ma, and H.~Zhang, ``Hyperspectral
  image restoration by tensor fibered rank constrained optimization and
  plug-and-play regularization,'' \emph{IEEE Transactions on Geoscience and
  Remote Sensing}, 2021.

\bibitem{yang2020remote}
J.-H. Yang, X.-L. Zhao, T.-H. Ma, Y.~Chen, T.-Z. Huang, and M.~Ding, ``Remote
  sensing images destriping using unidirectional hybrid total variation and
  nonconvex low-rank regularization,'' \emph{Journal of Computational and
  Applied Mathematics}, vol. 363, pp. 124--144, 2020.

\bibitem{donoho2006most}
D.~L. Donoho, ``For most large underdetermined systems of linear equations the
  minimal $\ell_1$-norm solution is also the sparsest solution,''
  \emph{Communications on Pure and Applied Mathematics: A Journal Issued by the
  Courant Institute of Mathematical Sciences}, vol.~59, no.~6, pp. 797--829,
  2006.

\bibitem{peng2014decomposable}
Y.~Peng, D.~Meng, Z.~Xu, C.~Gao, Y.~Yang, and B.~Zhang, ``Decomposable nonlocal
  tensor dictionary learning for multispectral image denoising,'' in
  \emph{Proceedings of the IEEE Conference on Computer Vision and Pattern
  Recognition}, 2014, pp. 2949--2956.

\bibitem{geman1992constrained}
D.~Geman and G.~Reynolds, ``Constrained restoration and the recovery of
  discontinuities,'' \emph{IEEE Transactions on Pattern Analysis and Machine
  Intelligence}, vol.~14, no.~3, pp. 367--383, 1992.

\bibitem{nikolova2005analysis}
M.~Nikolova and M.~K. Ng, ``Analysis of half-quadratic minimization methods for
  signal and image recovery,'' \emph{SIAM Journal on Scientific computing},
  vol.~27, no.~3, pp. 937--966, 2005.

\bibitem{boyd2011distributed}
S.~Boyd, N.~Parikh, E.~Chu, B.~Peleato, and J.~Eckstein, ``Distributed
  optimization and statistical learning via the alternating direction method of
  multipliers,'' \emph{Foundations and Trends{\textregistered} in Machine
  Learning}, vol.~3, no.~1, pp. 1--122, 2011.

\bibitem{PAM2014}
J.~Bolte, S.~Sabach, and M.~Teboulle, ``Proximal alternating linearized
  minimization for nonconvex and nonsmooth problems,'' \emph{Mathematical
  Programming}, vol.~46, no.~1, pp. 459--494, 2014.

\bibitem{bao2015dictionary}
C.~Bao, H.~Ji, Y.~Quan, and Z.~Shen, ``Dictionary learning for sparse coding:
  Algorithms and convergence analysis,'' \emph{IEEE Transactions on Pattern
  Analysis and Machine Intelligence}, vol.~38, no.~7, pp. 1356--1369, 2015.

\bibitem{aharon2006k}
M.~Aharon, M.~Elad, and A.~Bruckstein, ``K-{SVD}: An algorithm for designing
  overcomplete dictionaries for sparse representation,'' \emph{IEEE
  Transactions on signal processing}, vol.~54, no.~11, pp. 4311--4322, 2006.

\bibitem{cai2010singular}
J.-F. Cai, E.~J. Cand{\`e}s, and Z.~Shen, ``A singular value thresholding
  algorithm for matrix completion,'' \emph{SIAM Journal on Optimization},
  vol.~20, no.~4, pp. 1956--1982, 2010.

\bibitem{PAMsequrence}
H.~Attouch, J.~Bolte, and B.~F. Svaiter, ``Convergence of descent methods for
  semi-algebraic and tame problems: Proximal algorithms, forward-backward
  splitting, and regularized gauss$-$seidel methods,'' \emph{Mathematical
  Programming}, vol. 137, no.~1, pp. 91--129, 2013.

\bibitem{rockafellar2009variational}
R.~T. Rockafellar and R.~J.-B. Wets, \emph{Variational analysis}.\hskip 1em
  plus 0.5em minus 0.4em\relax Springer Science \& Business Media, 2009, vol.
  317.

\bibitem{clarke2008nonsmooth}
F.~H. Clarke, Y.~S. Ledyaev, R.~J. Stern, and P.~R. Wolenski, \emph{Nonsmooth
  analysis and control theory}.\hskip 1em plus 0.5em minus 0.4em\relax Springer
  Science \& Business Media, 2008, vol. 178.

\bibitem{PAM}
H.~Attouch, B.~Jerome, P.~Redont, and A.~Soubeyran, ``Proximal alternating
  minimization and projection methods for nonconvex problems. an approach based
  on the {K}urdyka-{L}ojasiewicz inequality,'' \emph{Mathematics of Operations
  Research}, vol.~35, no.~2, pp. 438--457, 2010.

\bibitem{semitoKL}
J.~Bolte, A.~Daniilidis, A.~Lewis, and M.~Shiota, ``Clarke subgradients of
  stratifiable functions,'' \emph{SIAM Journal on Optimization}, vol.~18,
  no.~2, pp. 556--572, 2007.

\bibitem{yair2018multi}
N.~Yair and T.~Michaeli, ``Multi-scale weighted nuclear norm image
  restoration,'' in \emph{Proceedings of the IEEE Conference on Computer Vision
  and Pattern Recognition}, 2018, pp. 3165--3174.

\bibitem{ssim}
Z.~Wang, A.~C. Bovik, H.~R. Sheikh, and E.~P. Simoncelli, ``Image quality
  assessment: from error visibility to structural similarity,'' \emph{IEEE
  Transactions on Image Processing}, vol.~13, no.~4, pp. 600--612, 2004.

\bibitem{wang2002universal}
Z.~Wang and A.~C. Bovik, ``A universal image quality index,'' \emph{IEEE signal
  processing letters}, vol.~9, no.~3, pp. 81--84, 2002.

\bibitem{yuhas1993determination}
R.~H. Yuhas, J.~W. Boardman, and A.~F. Goetz, ``Determination of semi-arid
  landscape endmembers and seasonal trends using convex geometry spectral
  unmixing techniques,'' in \emph{JPL, Summaries of the 4 th Annual JPL
  Airborne Geoscience Workshop}, vol.~1, 1993, pp. 147--149.

\bibitem{yu2020dictionary}
Q.~Yu, W.~Dai, Z.~Cvetkovi{\'c}, and J.~Zhu, ``Dictionary learning with
  blotless update,'' \emph{IEEE Transactions on Signal Processing}, vol.~68,
  pp. 1635--1645, 2020.

\bibitem{tproduct2018lu}
C.~Lu, \emph{Tensor-Tensor Product Toolbox}, Carnegie Mellon University, June
  2018, https://github.com/canyilu/tproduct.

\end{thebibliography}
}
\begin{IEEEbiography}[{\includegraphics[width=1in,height=1.25in,clip,keepaspectratio]{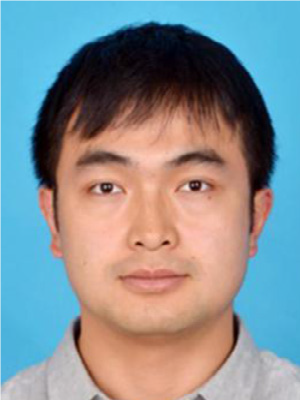}}]{Tai-Xiang Jiang} Tai-Xiang Jiang received the Ph.D. degrees in mathematics from the University of Electronic Science and Technology of China (UESTC), in 2019. He was a co-training Ph.D. student in the University of Lisbon supervised by Prof. Jose M. Bioucas-Dias from 2017 to 2018. He was the research assistant in the Hong Kong Baptist University supported by Prof. Michael K. Ng in 2019. He is currently an Associated Professor with the School of Economic Information Engineering, Southwestern University of Finance and Economics. His research interests include sparse and low-rank modeling and tensor decomposition for multi-dimensional image processing, especially on the low-level inverse problems for multi-dimensional images. https://sites.google.com/view/taixiangjiang/
\end{IEEEbiography}

\begin{IEEEbiography}[{\includegraphics[width=1in,height=1.25in,clip,keepaspectratio]{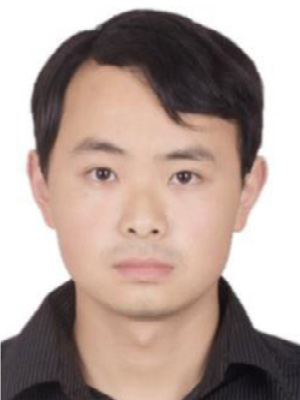}}]{Xi-Le Zhao}
 received the M.S. and Ph.D. degrees from the University of Electronic Science and Technology of China (UESTC), Chengdu, China, in 2009 and 2012.
He is currently a Professor with the School of Mathematical Sciences, UESTC. His main research interests are focused on the models and algorithms of high-dimensional image processing problems.
\end{IEEEbiography}

\begin{IEEEbiography}[{\includegraphics[width=1in,height=1.25in,clip,keepaspectratio]{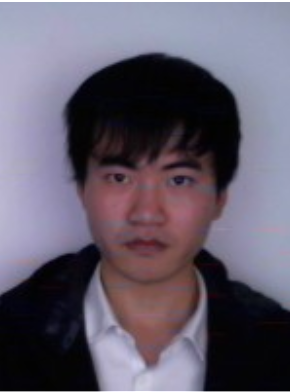}}]{Hao Zhang}
received the B.S. degrees in the college of science from the Sichuan Agricultural University (SAU), Ya’an, China, in 2018. He is currently pursuing the M.S. degree with the School of Mathematical Sciences, University of Electronic Science and Technology of China (UESTC), Chengdu, China. His research interests are modeling and algorithm for high-order data recovery based on the low rank prior of  tensors.
\end{IEEEbiography}

\begin{IEEEbiography}[{\includegraphics[width=1in,height=1.25in,clip,keepaspectratio]{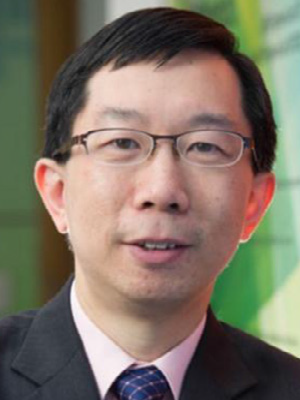}}]{Michael K. Ng}  is the Director of Research Division for Mathematical and Statistical Science, and Chair Professor of Department of Mathematics, the University of Hong Kong,  and Chairman of HKU-TCL Joint Research Center for AI. His research areas are data science, scientific computing, and numerical linear algebra.
\end{IEEEbiography}
\end{document}